\documentclass{article}

\usepackage[final,nonatbib]{neurips_2023}

\usepackage[utf8]{inputenc} 
\usepackage[T1]{fontenc}    
\usepackage{url}            
\usepackage{booktabs}       
\usepackage{amsfonts}       
\usepackage{nicefrac}       
\usepackage{microtype}      
\usepackage{xcolor}         

\usepackage[normalem]{ulem}
\usepackage{microtype}
\usepackage{graphicx}
\usepackage{subcaption}
\usepackage{enumerate}
\usepackage{times}
\usepackage{booktabs} 
\usepackage{bm}
\usepackage{diagbox}
\usepackage{algorithm,algorithmic}
\usepackage{thm-restate}

\usepackage[algo2e,ruled,vlined]{algorithm2e}

\usepackage{amsmath,amsthm,amssymb,mathtools}
\usepackage{thmtools, thm-restate}
\usepackage{tikz}
\usepackage{url}
\usepackage{setspace}
\usepackage[pdftex,bookmarksnumbered,bookmarksopen,
colorlinks,citecolor=blue,linkcolor=blue,urlcolor=blue]{hyperref}

\usepackage{xcolor}
\usepackage{soul}
\usepackage{longtable}

\usepackage{times}
\usepackage{enumitem}
\usepackage{varwidth}
\usepackage{graphicx}
\usepackage{wrapfig}
\usepackage{enumerate}
\usepackage{caption}

\usepackage{amssymb}
\usepackage{graphicx}
\usepackage{url}
\usepackage{setspace}
\usepackage{framed}
\usepackage{xcolor}
\usepackage{soul}
\usepackage{longtable}

\usepackage{times}
\usepackage{enumitem}
\usepackage{varwidth}
\usepackage{graphicx}
\usepackage{wrapfig}
\usepackage{enumerate}

\usepackage[utf8]{inputenc} 
\usepackage[T1]{fontenc}    
\usepackage{url}            
\usepackage{booktabs}       
\usepackage{amsfonts}       
\usepackage{nicefrac}       
\usepackage{microtype}      

\usepackage{graphicx}
\usepackage{appendix}
\usepackage{mdwlist}
\usepackage{xspace}
\usepackage{color}
\usepackage{mathrsfs}

\usepackage{booktabs}
\usepackage{comment}

\usepackage{multirow}

\newcommand{\sign}{\operatorname{sign}}

\newcommand{\err}{\mathrm{err}}

\newcommand{\error}{\mathrm{err}}

\newcommand{\Appendix}[1]{the full version for}

\newtheorem{theorem}{Theorem}[section]
\newtheorem{lemma}[theorem]{Lemma}

\newtheorem{assumption}{Assumption}
\newtheorem{definition}{Definition}

\newcommand{\x}{\bm{x}}

\newcommand{\z}{\mathbf{z}}

\newcommand{\B}{\mathbf{B}}

\newcommand{\cM}{\mathcal{M}}

\newcommand{\R}{\mathbb{R}}

\newcommand{\Z}{\mathbf{Z}}

\renewcommand{\comment}[1]{}
\newcommand{\red}[1]{}
\newcommand{\blue}[1]{{\color{black}#1}}

\newcommand{\cready}[1]{{\color{black}#1}}
\newcommand{\edit}[1]{{\color{black}#1}}

\newcommand{\cD}{\mathcal{D}}

\newcommand{\cH}{\mathcal{H}}

\newcommand{\cL}{\mathcal{L}}

\newcommand{\cU}{\mathcal{U}}

\newcommand{\cO}{\mathcal{O}}
\newcommand{\cP}{\mathcal{P}}
\newcommand{\cQ}{\mathcal{Q}}

\newcommand{\cY}{\mathcal{Y}}
\newcommand{\cX}{\mathcal{X}}

\newcommand{\bbI}{\mathbb{I}}

\newcommand{\ptoq}{\cP {\to} \cQ}

\newcommand{\ellca}{\ell_{\text{CA}}}
\newcommand{\elltl}{\ell_{\text{TL}}}
\newcommand{\ellst}{\ell_{\text{ST}}}
\newcommand{\ellia}{\ell_{\text{IA}}}
\newcommand{\rrca}{\text{RR}_{\text{CA}}^{\mathcal{L}}}
\newcommand{\rrtl}{\text{RR}_{\text{TL}}^{\mathcal{L}}}
\newcommand{\rrst}{\text{RR}_{\text{ST}}^{\mathcal{L}}}
\newcommand{\rria}{\text{RR}_{\text{IA}}^{\mathcal{L}}}
\newcommand{\srca}{\text{SR}^{\mathcal{L}}_{\text{CA}}(S,\eta_1,\eta_2)}
\newcommand{\srtl}{\text{SR}^{\mathcal{L}}_{\text{TL}}(S,\eta_1,\eta_2)}
\newcommand{\srst}{\text{SR}^{\mathcal{L}}_{\text{ST}}(S,\eta_1,\eta_2)}
\newcommand{\sria}{\text{SR}^{\mathcal{L}}_{\text{IA}}(S,\eta_1,\eta_2)}

\definecolor{colorY}{rgb}{0.7 , 0.7 , 0.2}

\date{}
\title{Reliable learning in challenging environments}

\author{
  Maria-Florina Balcan\\
  Carnegie Mellon University\\
  \texttt{ninamf@cs.cmu.edu} \\
  \And
  Steve Hanneke\\
  Purdue University\\
  \texttt{steve.hanneke@gmail.com} \\
  \AND
  Rattana Pukdee\\
  Carnegie Mellon University\\
  \texttt{rpukdee@cs.cmu.edu} \\
  \And
  Dravyansh Sharma\\
  Carnegie Mellon University\\
  \texttt{dravyans@cs.cmu.edu} \\
}\date{}

\DeclareMathOperator*{\argmin}{argmin}

\begin{document}
\maketitle

\newenvironment{proofoutline}{\noindent{\emph{Proof Sketch. }}}{\hfill$\square$\medskip}

\begin{abstract}%
\blue{The problem of designing learners that provide guarantees that their predictions are provably correct is of increasing importance in machine learning. However, learning theoretic guarantees have only been considered in very specific settings.  In this work, we consider the design and analysis of reliable learners in challenging test-time environments as encountered in modern machine learning problems: namely `adversarial' test-time attacks (in several variations) and `natural' distribution shifts.  In this work, we provide a reliable learner with provably optimal guarantees in such settings. We discuss \edit{practical} implementations of the learner and further show that our algorithm achieves strong positive performance guarantees on several natural examples: for example, linear separators under log-concave distributions or smooth boundary classifiers under smooth probability distributions.
}

\end{abstract}

\section{Introduction}
The question of providing  reliability guarantees on the output of learned classifiers has been studied previously 
in the classical learning setting where the training and test data are independent and identically distributed (i.i.d.) draws from the same distribution \cite{rivest1988learning,el2010foundations,el2012active}.  \blue{Conceptually,  a {\it reliable} learner outputs a prediction and may output a correctness guarantee. We know that the learner is correct on all points with the guarantee as long as the learning-theoretic assumptions hold, e.g., realizability. While a trivial model that abstains from providing any guarantee is also a reliable learner, we are interested in a reliable learner that provides the guarantee on as many points as possible ({\it useful} in the sense of Rivest and Sloan \cite{rivest1988learning}). \cite{el2010foundations} provides a characterization of optimal reliable learners in this classical learning setting.
}

 \blue{However, the assumption that the training and test data are drawn from the same distribution is often violated in practice. The mismatch may take the form of a `natural distribution shift' when the test distribution is different from the training distribution or `adversarial attacks' when there is an adversary that can perturb a test data point with the goal of changing the model prediction.} This is frequently accompanied by a significant performance drop, as well as the inability to guarantee the usefulness of the algorithm. As a result, there is a significant interest in the study of test-time attacks \cite{goodfellow2014explaining,carlini2017towards,madry2018towards} and distribution shift \cite{lipton2018detecting,recht2019imagenet,miller2021accuracy} among the applied machine learning community. Furthermore, recently there has been growing interest in the theoretical machine learning community for designing approaches with provable guarantees under test time attacks \cite{attias2019improved,montasser2019vc,montasseradversarially} as well as renewed interest in distribution shift \cite{ben2006analysis,mansour2008domain,hanneke2019value}.
All the prior theoretical work in the literature has mainly focused on the effect of attacks or distribution shift on average error rate {(}e.g.\ \cite{ben2006analysis,attias2019improved}{)}. However, this neglects a major relevant concern for users of machine learning algorithms, namely the ability to provide correctness guarantees for individual predictions: i.e., reliability. In {{this work}}, we advance this line of work by  developing a general understanding of how to learn reliably in the presence of {corruptions or changes to the test set, specifically under adversarial} test-time attacks { as well as} distribution shift {between the training (source) and test (target) data}.

\textbf{Our results}. We consider algorithms that provide robustly-reliable predictions which are guaranteed to be correct under standard assumptions from statistical learning theory, for both test-time attacks and distribution shift. Our first main set of results tackles the challenging case of {adversarial} test-time { perturbations}. For this setting, we introduce a novel compelling  
reliability criterion on a learner that  particularly captures the challenge of reliability under the test-time attacks. Given a test point $z$, a {{\it robustly-reliable} } classifier {either abstains from prediction, or} outputs both a prediction $y$ and a reliability guarantee $\eta$ with the guarantee that $y$ is correct unless one of  two bad events has occurred: 1) the true target function does not belong to the given hypothesis set $\cH$ or, 2) a test-point $z$ is perturbed from its original point by adversarial strength of at least $\eta$ {(measured in the relevant metric)}.  In the case of distribution shift, we provide novel analysis and a complexity measure that extend the classical notion of reliable learning to the setting when the test distribution is allowed to be an arbitrary new distribution. 

\subsection{Summary of contributions}
\begin{enumerate}[leftmargin=*,topsep=0pt,partopsep=1ex,parsep=1ex]\itemsep=-4pt

    \item We propose robustly-reliable learners for  test-time attacks which guarantee reliable learning in the presence of test-time attacks, and characterize the region of instance space where they are simultaneously robust and reliable. 
    Specifically, under the realizable setting, for adversarial perturbations within metric balls around the test points, we use the radius of the metric ball as a natural notion of adversarial strength. We output a reliability radius $\eta$ with a guarantee that our prediction on a point is correct as long as it was perturbed with a distance less than $\eta$ (under a given metric). We further show that our proposed robustly-reliable learner achieves pointwise optimal values for this reliability radius: that is, no robustly-reliable learner can output a reliability radius larger than our learner for any point in the instance space \blue{(Theorem \ref{thm:rr-eta-l1}, \ref{thm:mball-l3-ub})}. 

    \item \blue{The pointwise optimal algorithm is easy to derive from our definition. We discuss a \edit{practical} implementation of the optimal learners. (Section \ref{subsec: computational efficiency}). }

    \item We discuss variants of these algorithms and guarantees appropriate for three different variants of adversarial losses studied in the literature: depending on whether the perturbed point must have the same label as the original point, or in lieu of this, whether the algorithm should predict the true label of the perturbed point, or the same label as the original point \blue{(Definition \ref{def:losses})}.

    \item \blue{We further introduce a safely-reliable region, which captures the challenge caused by the adversary's ability to perturb a test point to cause a reduction in our reliability radius (Definition \ref{def:srr}). As examples, we show that the safely-reliable region can be large for linear separators under log-concave distributions and for classifiers with smooth decision boundaries under nearly-uniform distributions and as a consequence, the robustly-reliable region is large as well (Theorem \ref{thm: SR-l1l2l3}).}

    \item  We extend this characterization to abstention-based reliable predictions for arbitrary adversarial perturbation sets, where we no longer restrict ourselves to metric balls. We again get a tight characterization of the robustly-reliable region (\blue{Theorem \ref{thm: rrr-l3-lb-general}}).

    \item We also consider reliability in the distribution shift setting where the test data points come from a different distribution. We introduce a novel refinement to the notion of disagreement coefficient  \cite{hanneke2007bound}, to measure the \textbf{transferability of reliability  guarantees} across distributions. We provide bounds on the probability mass of the reliable region under transfer for several interesting examples including, when learning linear separators, transfer from $\beta_1$ log-concave to $\beta_2$ log-concave and to $s$-concave distributions (Theorems \ref{thm:pq-beta-concave}, \ref{thm:pq-s-concave}). We additionally bound the probability of the reliable region for learning classifiers with general smooth classification boundaries, for transfer between smooth distributions (Theorem \ref{thm:pq-smooth}). 

    \item We further extend our reliability results to the setting of robustness tranfer, where the test data is simultaneously under adversarial perturbations as well as distribution shift (Section \ref{sec:robustness-transfer}).
    \item \cready{Finally, we demonstrate that it is possible extend our results into the agnostic setting. (Section \ref{sec:agnostic})}

\end{enumerate}
\textbf{Conceptual advances over prior work.}

\cready{Prior works on certified robustness \cite{steinhardt2017certified,cohen2019certified,wang2022improved} have examined pointwise consistency guarantees. 
The certified robustness guarantee is only that a prediction does not change with an adversarial perturbation, but it does not guarantee that the prediction is correct (neither for the original point nor the perturbation); in particular, a constant function is always certified robust but it may not be useful. In contrast, our notion of robustly-reliable learner guarantees that, for any test point $x$ and perturbation $z$, if $z$ has a distance less than $\eta$ to $x$ ($\eta$ = reliability radius), then the prediction will be “correct” (robust loss zero) in a sense informed by which robust loss we are addressing; we discuss this idea for several different losses, leading to different interpretations of this guarantee.  In particular, for the stability loss, the prediction being “correct” means that it predicts the true label of the original point $x$; this implies certified robustness, but is even stronger, since it also guarantees the correct label.  Prior work \cite{balcan2022robustly} introduces the notion of a robustly-reliable learner for poisoning attacks which is different from our definition that is tailored to test-time attacks with a guarantee in terms of a reliability radius. In distribution shifts setting, we are the first to assess the transferability of reliability guarantees which differ from a widely-studied metric of average error rate. For additional related work, we refer to Appendix \ref{appendix: related work}.}


\section{Preliminaries and problem formulation}

Let $\cX$ denote the instance space and $\cY=\{0,1\}$ be the label space. Let $\cH$ be a hypothesis class. 
The learner $\cL$ is given access to a labeled sample $S=\{(x_i,y_i)\}_{i=1}^m$ drawn from a distribution $\cD$ over $\cX \times \cY$ and learns a concept $h^{\cL}:\cX\rightarrow\cY$. In the realizable setting, we assume we have a hypothesis (concept) class $\cH$ and target concept $h^*\in\cH$ such that the {\it true label} of any $x\in\cX$ is given by $h^*(x)$. In particular, $S=\{(x_i,h^*(x_i))\}_{i=1}^m$ in this setting. Given the 0-1 loss function $\ell:\cH\times\cX\rightarrow \{0,1\}$, define $\operatorname{err}_S(h,\ell)=\frac{1}{m}\sum_{(x,y)\in S}\ell(h,x)$. We use  $\cD_\cX$ to denote the marginal distribution over $\cX$. We use $\bbI[\cdot]$ to denote the indicator function that takes values in $\{0,1\}$. We also define $B_\cD^{\cH}(h^*,r) = \{h \in \cH \mid \Pr_\cD[h(x) \neq h^*(x)]\le r\}$ as the set of hypotheses in $\cH$ that disagree with $h^*$ with probability at most $r$. During test-time, the learner makes a prediction on a test-point $z\in \cX$. We consider the following settings

\begin{enumerate}[leftmargin=*,topsep=0pt,partopsep=1ex,parsep=1ex]\itemsep=-4pt
    \item \textbf{Adversarial test-time attack.} We consider adversarial attacks with perturbation function $\cU:\cX\rightarrow 2^{\cX}$ that can perturb a test point $x$ to an arbitrary point $z$ from the perturbation set $\cU(x)\subseteq \cX$ \cite{montasser2019vc}. We assume that the adversary has access to the learned concept $h^{\cL}$ as well as the test point $x$, and can perturb this data point to any $z \in \cU(x)$ and then provide this perturbed data point to the learner at test-time. We want to provide pointwise robustness and reliability guarantees in this setting. We will assume that $x\in\cU(x)$ for all $x\in\cX$. For any point $z$, we have $\cU^{-1}(z):= \{x \in \cX | z \in \cU(x)\}$, the set of points that can be perturbed to $z$. We use  {\it perturbation} to refer to a point $z \in \cU(x)$ and the perturbation sets $\cU(x)$ interchangeably.

    \item \textbf{Distribution shift.} We consider when a test point $z$ is drawn from a different distribution from the training samples. In this case, we want to provide a pointwise reliability guarantee. We will discuss more on this in Section \ref{sec: distribution shift}.
\end{enumerate}

\subsection{Robust loss functions}
In the applied and theoretical literature, various definitions of adversarial success have been explored, each dependent on the interpretation of robustness; depending on whether the perturbed point must have the same label as the original point, or in lieu of this, whether the algorithm should predict the true label of the perturbed point, or the same label as the original point. To capture these, we formally consider the following loss functions.
\begin{figure}[ht]
    \centering
    \includegraphics[width =0.85\textwidth]{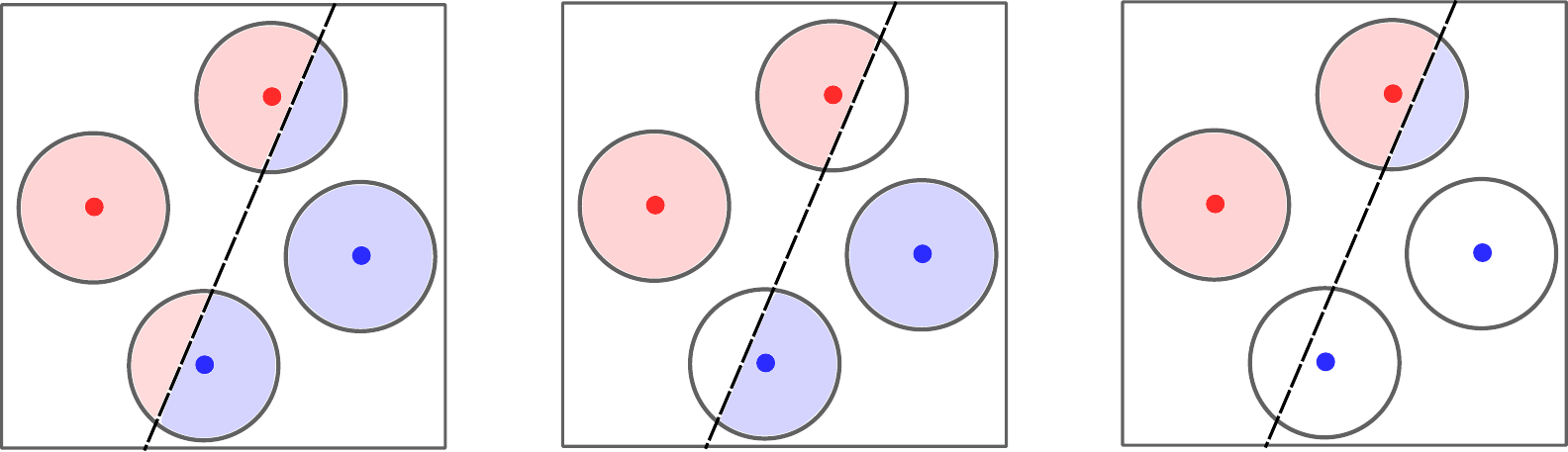}
    
    \caption{\blue{Different perturbation sets considered in  $\elltl,\ellst$ (left), $\ellca$ (mid) and $\ellia$ (right). The dashed line represents the decision boundary of $h^*$ and the background color of red and blue represents the label $0$ and $1$ respectively. The ball around each point describes the possible perturbation set $\cU(x)$ and the shaded area inside each ball is the allowed perturbation. In $\elltl,\ellst$, we consider all perturbation in $\cU(x)$ while in $\ellca$, we consider perturbations that do not change the true label of the perturbed point. Lastly, in $\ellia$, an adversary only perturb points where the original true label is $0$.}}
    \label{fig: Different perturbation sets}
\end{figure} 

\begin{definition}[Robust loss functions] \label{def:losses}
For a hypothesis $h$, a test point $x$, and a perturbation function $\cU$, we consider the following adversarially successful events.

\begin{enumerate}[leftmargin=*,topsep=0pt,partopsep=1ex,parsep=1ex]\itemsep=-4pt

    \item  \text{\textbf{Constrained Adversary loss}} \cite{szegedy2013intriguing,balcan2020power}. There exists a perturbation $z$ of $x$ that does not change the true label of an original point $x$ but $h(z)$ is incorrect. 
    \begin{equation*}
        \ellca^{h^*}(h,x) = \sup_{\substack{z\in\cU(x)\\
        h^*(z) = h^*(x)}} \bbI[h(z) \neq h^*(z)].
    \end{equation*} 
    For a fixed perturbation $z\in\cU(x)$, define $\ellca^{h^*}(h,x,z)=\bbI[h(z) \neq h^*(z) \wedge h^*(z) = h^*(x)]$.

  \item \textbf{True Label loss} \cite{zhang2019defending, gourdeau2021hardness}. There exists a perturbation $z$ of $x$ such that $h(z)$ is incorrect.   
    \begin{equation*}
        \elltl^{h^*}(h,x) = \sup_{z\in\cU(x)}\bbI[h(z) \neq h^*(z)].
    \end{equation*} 
    In this case, if the true label of the $z$ changes, then the learner need to match its prediction with the new label. For a fixed perturbation $z\in\cU(x)$, define $\elltl^{h^*}(h,x,z)=\bbI[h(z) \neq h^*(z)]$.

    \item \textbf{Stability loss} \cite{attias2019improved, montasser2019vc,montasseradversarially}. There exists a perturbation $z$ of $x$ such that $h(z)$ is different from $h^*(x)$.

    \begin{equation*}
        \ellst^{h^*}(h,x) = \sup_{z\in\cU(x)}\bbI[h(z) \neq h^*(x)].
    \end{equation*}
    In this case, we focus on the consistency aspect where we want the prediction of any perturbation $z$ to be the same as the prediction of $x$ and this has to be correct w.r.t. $x$ i.e. equals to $h^*(x)$. For a fixed perturbation $z\in\cU(x)$, define $\ellst^{h^*}(h,x,z)=\bbI[h(z) \neq h^*(x)]$. 
       \item  \blue{\text{\textbf{Incentive-aware Adversary loss}} \cite{zhang2021incentive}. We take inspiration from economics application where we assume that the label $1$ is a more favorable outcome e.g. loan approval for which an adversary has no incentive to make any perturbation when the original label is $1$. Define the perturbation set 
       \begin{equation*}
        \cU_{\text{IA}}(x,h^*) = \begin{cases}
        \cU(x) &\text{; $h^*(x) = 0$}\\
        \{x\} &\text{; $h^*(x) = 1$}
        \end{cases}
        \end{equation*}
    and define an incentive-aware adversary loss as
       \begin{equation*}
        \ellia^{h^*}(h,x) = \sup_{\substack{z\in\cU_{\text{IA}}(x,h^*)}} \bbI[h(z) \neq h^*(x)].
    \end{equation*} 
      For a fixed perturbation $z \in \cU(x)$, define $\ellia^{h^*}(h,x,z) = \bbI[h(z) \neq h^*(x) \wedge z \in \cU_{\text{IA}}(x,h^*)]$.}

\end{enumerate}
We say that $h$ is robust to a perturbation function $\cU$ at $x$ w.r.t. a robust loss $\ell$ if $\ell^{h^*}(h,x) = 0$.
\end{definition}
\textbf{Remark. } \blue{$\ellst, \ellia$} are robust losses that we can always evaluate in practice on the training data since we are comparing $h(z)$ with $h^*(x)$ which is known to us on the training data. For $\ellca,\elltl$, we are comparing $h(z)$ with $h^*(z)$ for which $z$ may lie outside of the support of the natural data distribution and we may not have access to $h^*(z)$. We illustrate the relationship between these losses by making a few useful observations.
\begin{itemize}
[leftmargin=*,topsep=4pt,partopsep=1ex,parsep=1ex]\itemsep=-4pt
    \item In the robustly-realizable case \cite{montasser2020reducing} when the perturbation function $\cU$ does not change the \textit{true} label of any $x$ in the training or test data, then all the losses $\ellca,\elltl,\ellst$ are equivalent. This corresponds to a common assumption in the adversarial robustness literature, that the perturbations are ``human-imperceptible'', which is usually quantified as the set of perturbations within a small metric ball around the data point.
   
    \blue{
    \item We provide an illustration of the perturbation set considered in various robust losses in Figure \ref{fig: Different perturbation sets}. By considering these perturbation set, we have the following implication $\elltl \rightarrow \ellca$, $\ellst \rightarrow \ellca$ and $\ellst \rightarrow \ellia$ where $\ell_1 \rightarrow \ell_2$ means robustness w.r.t. $\ell_1$ implies robustness w.r.t. $\ell_2$.}

\end{itemize}

{\color{black}
\section{Robustly-reliable learners w.r.t. metric ball attacks}\label{sec: rr-w-ball}
Although our robust losses are defined for any general perturbation set, we first consider the case where the perturbation sets are balls in some metric space. Such attacks are widely studied in the literature, in particular, for balls with bounded $L_p$-norm. Moreover, the radius of the metric ball serves as a natural notion of adversarial strength that allows us to quantify the level of robustness. We will later (Theorem \ref{thm: rrr-l3-lb-general}) present results for general perturbation sets as well. 

Let $\cM=(\cX,d)$ be a metric space equipped with distance metric $d$. We use the notation $\B_\cM(x,r)=\{x'\in\cX\mid d(x,x')\leq  r\}$ (resp.\ $\B^o_\cM(x,r)=\{x'\in\cX\mid d(x,x') <  r\}$) to denote a closed (resp.\ open) ball of radius $r$ centered at $x$. We will sometimes omit the underlying metric $\cM$ from the subscript to reduce notational clutter. We formally define a metric ball attack as follows.

\begin{definition}
     \textbf{Metric-ball attacks}
    are defined as the class of perturbation functions $\cU_{\cM}=\{u_\eta:\cX\rightarrow 2^\cX\mid u_\eta(x)=\B_\cM(x,\eta)\}$, induced by the metric $\cM=(\cX,d)$ defined over the instance space. 
\end{definition}

\noindent At test-time, given a test-point $z\in \cX$, we would like to make a prediction at $z$ with a reliability guarantee. We consider this type of learner, a {\it robustly-reliable} learner defined formally as follows.

\begin{definition}[Robustly-reliable learner w.r.t.\ $\cM$-ball attacks] \label{def:robustly-reliable-metric}
A learner $\cL$ is robustly-reliable w.r.t.\ $\cM$-ball attacks for hypothesis space $\cH$ and robust loss function $\ell$ if, \textbf{for any target} concept $h^*\in\cH$, given $S$ labeled by $h^*$, the learner outputs functions $h^\cL_{S}:\cX \to \cY$ and $r^\cL_{S}:\cX \to [0,\infty)\cup\{-1\}$ such that for all $x,z\in \cX$ if $r^\cL_{S}(z)=\eta > 0$ and $z\in \B^o_\cM(x,\eta)$ then $\ell^{h^*}(h^\cL_{S},x,z)=0$. Further, if $r^\cL_{S}(z) = 0$, then $h^*(z)=h^\cL_{S}(z)$.
\end{definition}

\noindent Note that $\cL$ outputs a prediction and 
a real value $r$ (the ``reliability radius'') 
for any test input. $r=-1$ corresponds to abstention (even in the absence of perturbation) \blue{i.e. when the learner is incapable of giving a reliability guarantee for that prediction)},  and $r=\eta > 0$ is a guarantee from the learner that if the adversary's attack is in $\B^o_\cM(x,\eta)$ then we are correct i.e. if an adversary changes the original test point $x$ to $z$, the attack will not succeed if the adversarial budget is less than $\eta$. \blue{Lastly, when $r = 0$, the learner provides a guarantee that the learner's prediction at $z$ is correct.}

\begin{definition}
    [Robustly-reliable region w.r.t. $\cM$-ball attacks] For a robustly-reliable learner $\cL$ w.r.t.\ $\cM$-ball attacks for sample $S$, hypothesis space $\cH$ and robust loss function $\ell$ defined above, the robustly-reliable region of $\cL$ at a reliability level $\eta$ is defined as $\text{RR}^{\cL}(S,\eta)=\{x\in\cX\mid r^\cL_{S}(x) \ge \eta\}$ for sample $S$ and $\eta\ge 0$.
\end{definition}

 \noindent The robustly-reliable region contains all points with a reliability guarantee of at least $\eta$. We use $\text{RR}^{\cL}_{W}$ to denote  robustly-reliable regions with respect to losses $\ell_W$ for $W \in \{\text{CA}, \text{TL},\text{ST},\text{IA}\}$. A natural goal is to find a robustly-reliable learner $\cL$ that has the largest robustly-reliable region possible. \blue{First, we note that predictions that are known by the learner to be correct are still known to be correct even when the test points are attacked. Therefore, a test point $z$ lies in the robustly-reliable region w.r.t. $\ellca,\elltl$, as long as we can be sure that $h^{\cL}_S(z)$ is correct. This is equivalent to $z$ being classified perfectly, i.e. according to the true label. 
 Therefore, the robustly-reliable region w.r.t.\ $\ellca,\elltl$ is given by the agreement region of the version space, which is  the largest region where we can be sure of what the correct label is in the absence of any adversarial attack \cite{el2010foundations}. We recall the definition of version space \cite{mitchell1982generalization} and agreement region \cite{cohn1994improving,balcan2006agnostic}.}
}
\begin{definition}
For a set $H \subseteq \cH$ of hypothesis, and any set of samples $S$, let $\operatorname{DIS}(H) = \{x \in \cX: \exists h_1,h_2 \in \cH \text{ s.t. } h_1(x)\neq h_2(x)\}$ be the \textbf{disagreement region} and $\operatorname{Agree}(H) = \cX \setminus \operatorname{DIS}(H)$ be the \textbf{agreement region}. Let $\cH_0(S) = \{h \in \cH | \operatorname{err}_S(h) = 0\}$ be a \textbf{version space}: the set of all hypotheses that correctly classify $S$. More generally, $\cH_\nu(S) = \{h \in \cH | \operatorname{err}_S(h) \le \nu\}$ for $\nu\ge0$.

\end{definition}

\blue{
We can also characterize the robustly-reliable region with respect to other robust losses in terms of the agreement region in the following Theorem. 
\begin{theorem}\label{thm:rr-eta-universal}
Let $\cH$ be any hypothesis class. With respect to $\cM$-ball attacks and $\ell_W$, for $\eta\ge 0$, 
\begin{enumerate}[label=(\alph*),leftmargin=*,topsep=4pt,partopsep=1ex,parsep=1ex]\itemsep=-4pt
    \item there exists a robustly-reliable learner $\cL$ such that $\text{RR}_{W}^{\cL}(S,\eta)\supseteq A_W$, and
    \item for any robustly-reliable learner $\cL$, $\text{RR}_{W}^{\cL}(S,\eta)\subseteq A_W$.
\end{enumerate}
Specifically, for the robust loss $\ell_W$, the optimal robustly-reliable region $A_W$ are 
\begin{table}[h]
\centering
\begin{tabular}{ll}
\toprule
\textbf{Robust loss }$\ell_W$ &  \textbf{Optimal robustly-reliable region} $A_W$\\
\midrule
$\ellca, \elltl$ &  $\{z \mid z \in \operatorname{Agree}(\cH_0(S))\}$\\
$\ellst$ & $\{z \mid \B^o(z,\eta) \subseteq \operatorname{Agree}(\cH_0(S)) \wedge h(z) = h(x),\forall x \in \B^o(z,\eta), \forall h \in \cH_0(S) \}$ \\
$\ellia$ &  $(A_{\text{ST}} \cap \{z \mid h^*(z) = 1\})\cup \{ z \mid z \in \operatorname{Agree}(\cH_0(S)) \wedge h^*(z) = 0\}$\\
\bottomrule
\end{tabular}
\end{table}
\end{theorem}
}

\begin{proof} \cready{(Sketch) We provide the construction of the optimal robustly-reliable learner $\cL_{\text{opt}}$ such that $\text{RR}_{W}^{\cL_{\text{opt}}}(S,\eta)\supseteq A_W$ and later show that for any robustly-reliable learner $\cL$, we must also have $\text{RR}_{W}^{\cL}(S,\eta)\subseteq A_W$. We start with $\ellca, \elltl$, consider a learner $\cL_{\text{opt}}$ that predicts using an ERM classifier and outputs $\eta=\infty$ for all points in the agreement region of $\cH_0(S)$. Any prediction in $\operatorname{Agree}(\cH_0(S))$ is reliable because it also agrees with $h^*$ ($h^*\in \cH_0(S)$ by realizability). On the other hand, for $z \in \operatorname{DIS}(\cH_0(S))$, there exist $h_1,h_2 \in \cH_0(S)$ that disagree on $z$. For any learner $\cL$, it is not possible to guarantee that $h^{\cL}(z)$ is correct as we may have $h^* = h_1$ or $h^* = h_2$. 

Now, for $\ellst$, consider a learner $\cL_{\text{opt}}$ that classifies using an ERM but the reliability radius is now the largest $\eta> 0$ for which $\B^o(z,\eta)\subseteq\operatorname{Agree}(\cH_0(S)) $ and $ h(x)=h(z), \;\forall\;x\in \B^o(z,\eta), h\in \cH_0(S)$, else $\eta=0$ if $z\in \operatorname{Agree}(\cH_0(S))$, and $-1$ otherwise.  The first condition guarantees that $h(x) = h^*(x), \forall x \in \B^o(z,\eta)$. Combined with the second condition we have $h(z) = h(x) = h^*(x), \forall x \in \B^o(z,\eta)$. Thus,  $\cL_{\text{opt}}$ is a robustly-reliable learner. On the other hand, for a robustly-reliable learner $\cL$, consider $z \in \rrst(S,\eta)$ for $\eta>0$. We must have $h^{\cL}(z) = h^*(x), \forall x \in \B^o(z,\eta)$. 
    Using a similar argument to the case for  $\ellca, \elltl$, we have $z\in \operatorname{Agree}(\cH_0(S))$. If there exists $x\in \B^o(z,\eta)$ that $x\not\in \operatorname{Agree}(\cH_0(S))$, there exists $h_1,h_2\in \cH_0(S)$ that $h^{\cL}(z) \neq h_1(x)$ or $h^{\cL}(z) \neq h_2(x)$. It is not possible to guarantee that $h^{\cL}(z) = h^*(x)$ as we may have $h^* = h_1$ or $h^* = h_2$. Therefore, we must have $\B^o(z,\eta) \subseteq \operatorname{Agree}(\cH_0(S))$. Finally, we cannot have $x\in \B^o(z,\eta)$ that $h(z) = h^*(z) \neq h^*(x)$ since this contradict with $h(z) = h^*(x)$. Therefore, we must have $h^*(x) = h^*(z)$. Since we have $x \in \operatorname{Agree}(\cH_0(S))$, this implies that $h(x) = h(z)$ for $h \in \cH_0(S)$. For $\ellia$, the construction is similar to $\ellst$. For full proof, we refer to Appendix \ref{appendix: robustly-reliable metric ball}.}
\end{proof}

\blue{ For $\ellst$, the learner is able to certify a subset of the agreement region, which satisfies two additional conditions: {$h^{\cL}_S$ must be correct} on all possible points $x$ that could be perturbed to an observed test point $z$, and the true label of $z$ should match the true label of $x$. We denote the second condition of $ h(z) = h(x),\forall x \in \B^o(z,\eta), \forall h \in \cH_0(S)$ as the \textbf{label consistency condition}. For $\ellia$, since robustness w.r.t. $\ellst$ implies robustness w.r.t. $\ellia$, we have $A_\text{ST} \subset A_\text{IA}$. In addition, with an incentive-aware adversary, whenever $h(z) = 0$, we must have $h^*(x) = 0$ since the adversary does not perturb a data point with the original label $1$. Therefore, we can additionally provide a guarantee on $z$ that lies in the agreement region and $h(z) = 0$. We remark that we can identify the term $h^*(z)$ for any point $z \in \operatorname{Agree}(\cH_0(S))$ due to the realizability assumption.}

\noindent  We have provided a robustly-reliable learner with the largest possible robustly-reliable region for losses $\ellca, \elltl, \ellst, \ellia$. However, we note that the probability mass of the robustly-reliable region may not be a meaningful way to quantify the overall reliability of a learner because a perturbation $z$ may lie outside of the support of the natural data distribution and have zero probability mass. It seems more useful to measure the mass of points $x$ where any perturbation $z$ of $x$ still lies within the robustly-reliable region. We formally define this region as the safely-reliable region.

\begin{definition}
    [Safely-reliable region w.r.t. $\cM$-ball attacks] Let $\cL$ be a robustly-reliable learner w.r.t. $\cM$-ball attacks for sample $S$, hypothesis class $\cH$ and robust loss function $\ell$. The safely-reliable region of learner $\cL$ {at reliability levels $\eta_1, \eta_2$} is defined as 

    \begin{enumerate}[leftmargin=*,topsep=0pt,partopsep=1ex,parsep=1ex]\itemsep=-4pt
        \item $\srca= \{x \in \cX \mid  \B_\cM(x,\eta_1)\cap \{z\mid h^*(z) = h^*(x) \} \subseteq \rrca(S,\eta_2)\},$
        \item $\text{SR}^{\cL}_{W}(S,\eta_1,\eta_2) = \{x \in \cX \mid  \B_\cM(x,\eta_1) \subseteq \text{RR}^{\cL}_{W}(S,\eta_2)\}$ \text{ for } $W \in \{\text{TL},\text{ST} \}$,
        \item \blue{$\sria = \{x \in \cX \mid  h^*(x) = 0 \wedge \B_\cM(x,\eta_1) \subseteq \rria(S,\eta_2)\} \cup \{x \in \cX \mid  h^*(x) = 1 \wedge x \in  \rria(S,\eta_2))  \}$.}
    \end{enumerate}
    \label{def:srr}
\end{definition}

\red{[Optional TODO] add a remark comparing with \cite{balcan2022robustly}}

\red{[Optional Comment] An alternate presentation would be to state the definition of SR region in terms of the following property; and establish the containments in Definition \ref{def:srr}.}

\noindent The safely-reliable region contains any point that retains a reliability radius of at least $\eta_2$ even after being attacked by an adversary with strength $\eta_1$. {In the safely-reliable region, we consider a set of potential natural (before attack) points $x$, while in the robustly-reliable region, we consider a set of potential test points $z$.} 

In the following subsections, we show that in interesting cases commonly studied in the literature, the probability mass of the safely-reliable region is actually quite large.

\subsection{Safely-reliable region for linear separators under log-concave distributions is large}\label{sec: rrr-linear}
\red{[Optional TODO] For log-concave examples, our proof implies a non-optimal but easier to efficiently implement robustly-reliable learner, with similar bounds on the reliable regions -- maybe we should talk about that in our theorems.}

\noindent We provide the probability mass of safely-reliable regions with respect to different losses for linear separators when the data distribution follows an isotropic (mean zero and identity covariance matrix) log-concave (logarithm of density function is concave) distribution under a bounded $L_2$-norm ball attack. For full proof, we refer to Appendix  \ref{appendix: safely-reliable}. We will rely on the following key lemma \blue{which states that the agreement region of a linear separator cannot contain points that are {arbitrarily}  close to the decision boundary of $h^*$ for any sample $S$. }
\begin{lemma}
    Let $\cD$ be a distribution over $\R^d$ and $\cH = \{h: x \to \operatorname{sign}(\langle w_h, x\rangle ) \mid w_h \in \R^d, \lVert w_h \rVert_2 = 1\}$ be a class of linear separators. For $h^*\in \cH$, for a set of samples $S \sim \cD^m$ such that there is no data point in $S$ that lies on the decision boundary, for any $0 < c < d$, there exists $\delta(S,c,d) > 0$ such that for any $x$ with $ c \leq \|x\| \leq d$ and $|\langle w_{h^*}, x\rangle| < \delta$, we have $x\not\in \operatorname{Agree}(\cH_0(S))$.
    \label{lemma: agreement contain large margin}
\end{lemma}

\blue{A direct implication of the lemma is that any $L_2$-ball $\B(x,\eta)$ that lies in the agreement region must not contain the decision boundary of $h^*$ and must contain points with the same label. This allows us to remove the \textit{label consistency condition} and instead focus on whether the ball $\B(x,\eta)$ lies in the agreement region. Intuitively, the reliable region is now given by the `$\eta$-buffered' agreement region where we only select points that have a distance at least $\eta$ from the boundary of the agreement region (Figure \ref{fig: robustly-reliable regions}). We provide bounds for the probability mass of the safely-reliable region below and we refer to the full proof in Appendix \ref{appendix: safely-reliable}.

}

\begin{figure}
    \centering
\includegraphics[width = 0.8\columnwidth]{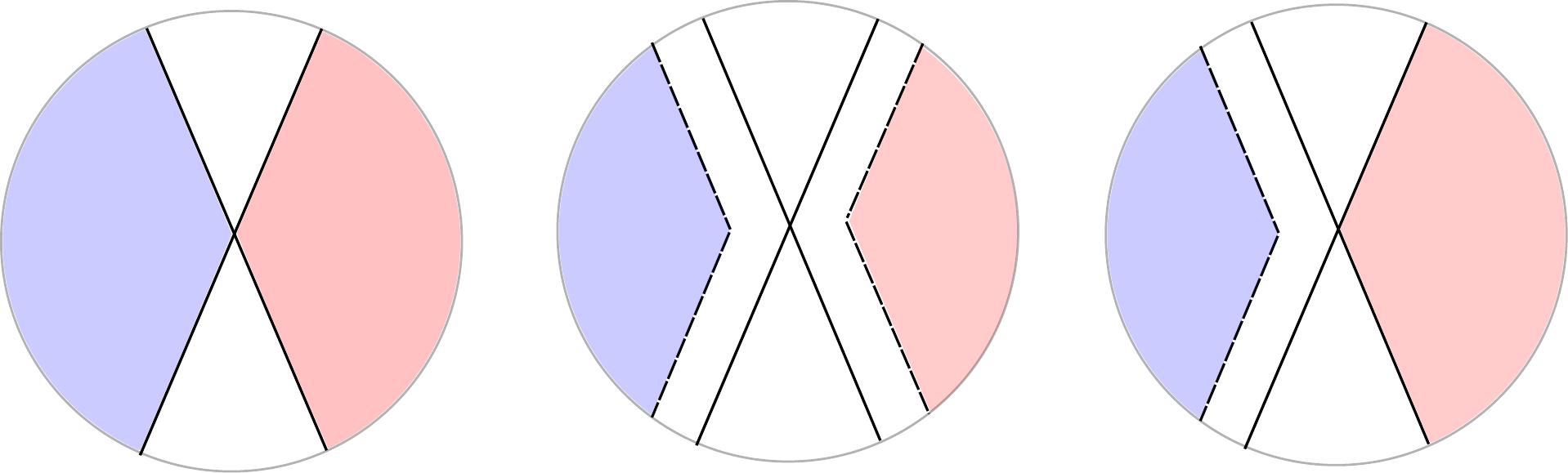}
    \caption{\blue{The robustly-reliable region for $\ellca, \elltl$ (left), $\ellst$ (mid) and $\ellia$(right) for linear separators with an $L_2$-ball perturbation. The background color of blue and red represents the agreement region of class $1$ and $0$ respectively. In this case, we can remove the label consistency condition and reduce the robustly-reliable region into the `$\eta$-buffered' agreement region.}}
    \label{fig: robustly-reliable regions}
\end{figure}

{\color{black}
\begin{theorem} Let $\cD$ be isotropic log-concave over $\R^d$ and $\cH = \{h: x \to \operatorname{sign}(\langle w_h, x\rangle ) \mid w_h \in \R^d, \lVert w_h \rVert_2 = 1\}$ be the class of linear separators. Let $\B(\cdot, \eta)$ be a $L_2$ ball perturbation with radius $\eta$. For $S \sim \cD^m$, for  $m = \cO(\frac{1}{\varepsilon^2}(\operatorname{VCdim}(\cH) + \ln\frac{1}{\delta}))$, for an optimal robustly-reliable learner $\cL$,
\begin{enumerate}[label=(\alph*),leftmargin=*,topsep=0pt,partopsep=1ex,parsep=1ex]\itemsep=-4pt
    \item $\Pr(\srtl) \geq 1 -  2\eta_1 - \Tilde{\cO}(\sqrt{d}\varepsilon)$ with probability at least $1 - \delta$,
    \item $\srca = \srtl$ almost surely,
    \item $\Pr(\srst) \geq  1 - 2(\eta_1 + \eta_2) - \Tilde{\cO}(\sqrt{d}\varepsilon)$ with probability at least $1 - \delta$,
    \item \blue{$\Pr(\sria) \geq  1 - (\eta_1 + \eta_2) - \Tilde{\cO}(\sqrt{d}\varepsilon)$ with probability at least $1 - \delta$.}
\end{enumerate}
The $\Tilde{\cO}$-notation suppresses dependence on logarithmic factors and distribution-specific constants.
\label{thm: SR-l1l2l3}
\end{theorem}
\blue{ We remark that we can't always remove the label consistency condition for a general perturbation set. For example, consider $\cU(x) = \B(x - a, \eta) \cup \{x\} \cup \B(x +a, \eta),$
is made of two $L_2$ balls with center $x-a, x+a$, with appropriate value of $a,\eta$, we may have each ball lie in the different side of the agreement region so that the whole perturbation set lie in the agreement region but contain points with different labels (Figure \ref{fig: different label}).
}
\blue{
We also provide bounds on the probability mass of the safely-reliable region for more general concept spaces beyond linear separators, specifically, classifiers with smooth boundaries in Appendix \ref{appendix: safely-reliable region for classifiers with smooth boundaries}.
}
\begin{figure}[h]
    \centering
\includegraphics[width = 0.265\columnwidth]{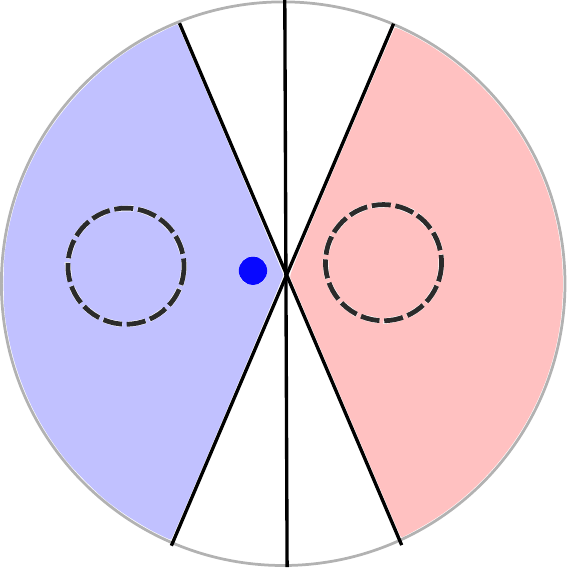}
    \caption{The perturbation set is represented by two dashed balls. This lies inside the agreement region but contains points with different labels.}
    \label{fig: different label}
\end{figure}

\section{\edit{Optimization-based implementation}}
\label{subsec: computational efficiency}
\blue{Given the definition, \edit{it is possible to implement robustly-reliable learners in practice.} For example, for linear separator concept classes under bounded $L_2$-norm attack. The optimal robustly-reliable learner \cready{$\cL_{\text{opt}}$}, described above may be implemented via a \edit{non-linear constrained optimization problem} that computes the (squared) distance of the test point $z$ to the closest point $z'$ in the disagreement region. This gives us the reliability radius, since for linear separators one must cross the decision boundary to perturb a point to a differently labeled point. \edit{Let $y_i \in \{-1,1\}$ be the label then the optimization problem is given by} 
\begin{align*}
    \min_{w,w',z'} ||z-&z'||^2\\
    \text{s.t.}\qquad\quad  \langle w,x_i\rangle y_i&\ge 0,\qquad \text{for each }(x_i,y_i)\in S,\\
     \langle w',x_i\rangle y_i&\ge 0,\qquad \text{for each }(x_i,y_i)\in S,\\
     \langle w,z'\rangle \langle w',z'\rangle &\le 0.\\
     \edit{\lVert w \rVert = \lVert w' \rVert } &\edit{= 1.}
\end{align*}

\edit{This formulation provides a concrete approach for seeking numerical solutions using general-purpose optimization solvers.} Given training sample $S$, for any given test point $z$, the learner $\cL$ can \edit{practically} compute the solution $s^*$ to the above program and output $\sqrt{s^*}$ as the \edit{lower bound of the true} reliability radius. We show that the variant of this objective also provides a reliability radius for a wide range of hypothesis classes under $L_2$ ball attacks \cready{(see Lemma \ref{lemma: reliability radius objective})}. In addition, we can relax this objective into a regularized objective that gives a lower bound on the reliability radius of $||z-z^*||^2$, when $z^*$ is the solution of
\begin{align*}
    h_1,h_2,z^* = \argmin_{h,h',z'} ||z-&z'||^2 + \lambda(\hat{R}(h, S \cup \{(z',0)\}) + \hat{R}(h', S \cup \{(z',1)\}))
\end{align*}
when $\hat{R}(h,S)$ is the empirical risk of $h$ on $S$. We provide a more detailed discussion in Appendix \ref{appendix: computational efficiency}.

}
\section{Robustly-reliable learning under distribution shift}
\label{sec: distribution shift}

\begin{figure*}[t]
    \centering
 \includegraphics[width=0.35\columnwidth]{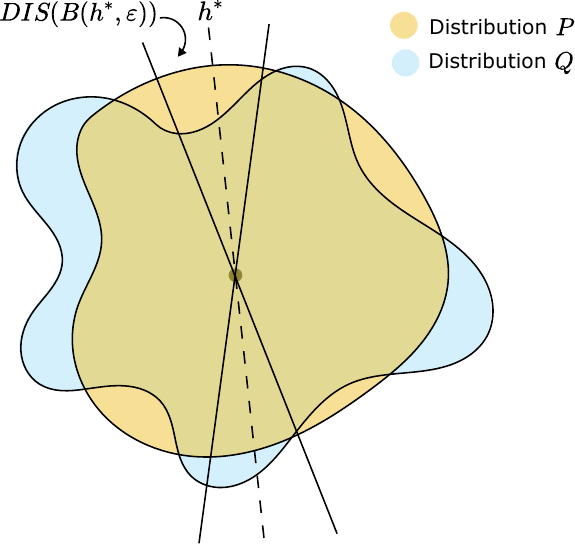}
 \quad\quad
    \includegraphics[width=0.35\columnwidth]{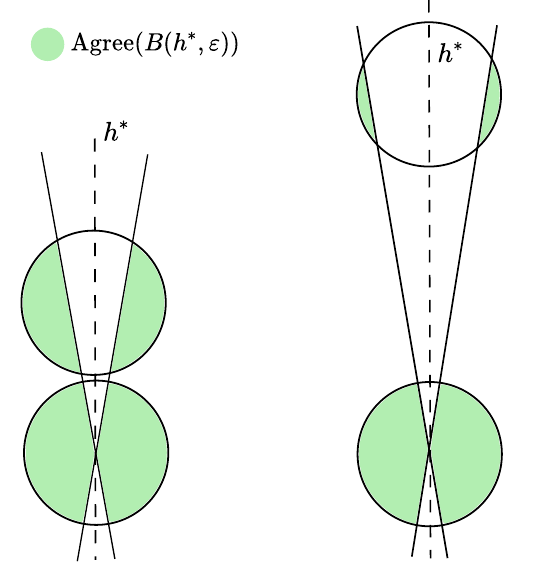}
    \caption{The disagreement region and the agreement region under a distribution shift where $\cP$ and $\cQ$ are isotropic (left) and where there is a mean shift (right).}
    \label{fig: distribution shift}
    \vspace{-4mm}
\end{figure*}

\noindent We now consider the reliability aspect for distribution shift, a different kind of test-time robustness challenge when the test data comes from a different distribution than the training data. Formally, let $\cP$ be the training distribution and let $\cQ$ be the test distribution. We assume the \textit{realizable distribution shift} setting, i.e. there is a target concept $h^* \in \cH$ such that the true label of any $x \in \cX$ is given by $h^*(x)$ at training time and test time, or $ \error_\cP(h^*) = \error_\cQ(h^*) = 0$. \blue{As observed earlier, points that are known by the learner to be correct (reliable) are still known to be correct even when it is drawn from a different distribution. } This reliability guarantee holds even when the distributions $\cP$ and $\cQ$ do not share a common support, a setting for which many prior theoretical works result in vacuous bounds. For example, suppose $\cX=\R^n$, $\cP$ and $\cQ$ are supported on disjoint $n$-balls, and $\cH$ is the class of linear separators. Then the total variation distance, the $\cH$-divergence \cite{kifer2004detecting,ben2010theory} as well as the discrepancy distance \cite{mansour2008domain} between $\cP$ and $\cQ$ are all 1. While recent work of \cite{hanneke2019value} does apply in this setting, they do not focus on the reliability guarantee. \blue{In this work, we are interested in quantifying the transferability of reliability guarantee transfer between distributions $\cP$ and $\cQ$. } We recall the notion of reliable prediction \cite{el2010foundations}. 
\begin{definition}[Reliability] 
\label{def:reliable}
A learner $\cL$ is reliable w.r.t.  concept space $\cH$ if, for any target concept $h^*\in\cH$, given any sample $S$ labeled by $h^*$, the learner outputs functions $h^\cL_{S}:\cX \to \cY$ and $a^\cL_{S}:\cX \to \{0,1\}$ such that for all $x\in \cX$ if $a^\cL_{S}(x) = 1$ then $h^\cL_{S}(x) = h^*(x)$. Else, if $a^\cL_{S}(x) = 0$, the learner abstains from prediction. The reliable region of  $\cL$ is  $R^{\cL}(S)=\{x\in\cX\mid a^\cL_{S}(x) = 1\}$.
\end{definition}

\noindent We define the following metric to measure the reliability of a learner under distribution shift.

\begin{definition}[$\ptoq$ reliable correctness] 
The $\ptoq$-reliable correctness of $\cL$ (at sample rate $m$, for distribution shift from $\cP$ to $\cQ$) \blue{is defined as the expected probability mass of its reliable region under $Q$ when trained on a random training $S\sim\cP^m$}, i.e. $\Pr_{x\sim Q, S\sim P^m}[x\in R^{\cL}(S)]$.

\end{definition}

{\color{black}
\noindent The disagreement coefficient was originally introduced to study the label complexity in agnostic active learning \cite{hanneke2007bound} and is also known to characterize reliable learning in the absence of any distribution shift \cite{el2010foundations}. We propose the following refinement to the notion of disagreement coefficient, which we will use to give bounds on a learner's $\ptoq$ reliable correctness. 

\begin{definition}[$\ptoq$ disagreement coefficient]\label{def-pq} For a hypothesis class $\cH$, the $\ptoq$ disagreement coefficient of $h^*\in \cH$ with respect to $\cH$ is given by
    \begin{equation*}
        \Theta_{\ptoq}(\varepsilon) =\sup_{r\ge \varepsilon}\frac{\Pr_\cQ[\operatorname{DIS}(B_\cP(h^*,r))]}{r},
    \end{equation*}
where $B_\cP(h^*,r) = \{h \in \cH \mid \Pr_\cP[h(x) \neq h^*(x)]\le r\} $.
\end{definition}

This quantifies the rate of disagreement over $\cQ$ among classifiers which are within disagreement-balls w.r.t. $h^*$ under $\cP$, relative to the version space radius. 
 The proposed metric is asymmetric between $\cP$ and $\cQ$, and also depends on the target concept $h^*$. More simple examples are in Appendix \ref{appendix: disagreement coeff}. We show that the $\ptoq$-reliable correctness of our learner may be bounded in terms of the $\ptoq$ disagreement coefficient using a uniform convergence based argument. The proof details are in Appendix \ref{appendix: distribution shift}.

\begin{theorem}\label{thm:pq-reliability}
Let $\cQ$ be a realizable distribution shift of $\cP$ with respect to $\cH$, and $h^*\in\cH$ be the target concept. 
Given sufficiently large sample size $m\ge\frac{c}{\varepsilon^2}(d+\ln\frac{1}{\delta})$, the $\ptoq$-reliable correctness of $\cL$, \blue{the optimal robustly-reliable learner}, is at least $$\Pr_{x\sim Q, S\sim P^m}[x\in R^\cL(S)]\ge 1-\Theta_{\cP \to \cQ}\cdot\varepsilon-\delta.$$ Here $c$ is an absolute constant, and $d$ 
is the VC-dimension of $\cH$.
\end{theorem}

\blue{In Appendix \ref{sec:robustness-transfer}, we show that this $\cP \to \cQ$ disagreement coefficient can be small for several examples which implies that it is possible to transfer the reliability guarantee from one distribution to the other. In particular, when learning linear separators, we provide bounds for transferring from $\beta_1$ log-concave to $\beta_2$ log-concave and to $s$-concave distributions (Theorems \ref{thm:pq-beta-concave}, \ref{thm:pq-s-concave}). In addition, when learning classifiers with general smooth classification boundaries, we provide bounds for transferring between smooth distributions (Theorem \ref{thm:pq-smooth}).

}

}
}

\section{Safely-reliable correctness under distribution shift \red{or Reliable robustness transfer}}
There is a growing practical \cite{shafahiadversarially,salman2020adversarially} as well as recent theoretical interest \cite{deng2023hardness} in the setting of `robustness transfer', where one simultaneously expects adversarial test-time attacks as well as distribution shift. We will study the reliability aspect for this more challenging setting. We note that the definition of a robustly-reliable learner does not depend on the data distribution (see Definition \ref{def:robustly-reliable-metric}) as the guarantee is pointwise. Our optimality result in Section \ref{sec: rr-w-ball} applies even when a test point is drawn from a different distribution $\cQ$. In this case, the safely-reliable region instead would have a different probability mass. 
\begin{definition}[$\ptoq$ safely-reliable correctness]\label{def:pqsrc}
The $\ptoq$ safely-reliable correctness of $\cL$ (at sample rate $m$, for distribution shift from $\cP$ to $\cQ$, w.r.t.\ robust loss $\ell$) is defined as the probability mass of its safely-reliable region under $Q$, on a sample $S\sim\cP^m$, i.e. $$\text{PQR}_\ell^\cL(S,\eta_1,\eta_2):=\Pr_{x\sim Q, S\sim P^m}[x\in \text{SR}^{\cL}_\ell(S, \eta_1,\eta_2)].$$
\end{definition}
\blue{
We consider an example when the training distribution $\cP$ is isotropic log-concave and the test distribution $\cQ_\mu$  is log-concave with its mean shifted by $\mu$ but the covariance matrix is still an identity matrix (see Figure \ref{fig: distribution shift}, right). We provide the bound on the $\ptoq$ safely-reliable correctness of this example in Appendix \ref{sec:robustness-transfer} (see Theorem \ref{thm: cross-proeduct linear separators}).
}
\cready{
\section{Reliability in the agnostic setting}\label{sec:agnostic}
In the above, we have assumed that the training samples $S$ are realizable under our concept class $\cH$, i.e.\ there is a target concept $h^*$ consistent with our (uncorrupted) data. In the agnostic setting, we can have $\min_{h\in\cH}\err_S(h)>0$, meaning no single concept is always correct. We define a {\it $\nu$-tolerably robustly-reliable} learner under test-time attacks in the agnostic setting as the learner whose reliable predictions agree with every low-error hypothesis (error at most $\nu$) on the training sample (\cite{balcan2022robustly} have proposed the corresponding definition for data poisoning attacks).

\begin{definition}[$\nu$-tolerably robustly-reliable learner w.r.t.\ $\cM$-ball attacks] \label{def:tolerably-robustly-reliable-metric}
A learner $\cL$ is robustly-reliable w.r.t. $\cM$-ball attacks for sample $S$, hypothesis space $\cH$ and robust loss function $\ell$ if, for \textbf{every} concept $h^*\in\cH$ with $\err_S(h^*)\le\nu$, the learner outputs functions $h^\cL_{S}:\cX \to \cY$ and $r^\cL_{S}:\cX \to [0,\infty)\cup\{-1\}$ such that for all $x,z\in \cX$ if $r^\cL_{S}(z)=\eta > 0$ and $z\in \B^o_\cM(x,\eta)$ then $\ell^{h^*}(h^\cL_{S},x,z)=0$. Further, if $r^\cL_{S}(z) = 0$, then $h^*(z)=h^\cL_{S}(z)$. Given sample $S$ such that some concept $h^*\in\cH$ satisfies $\err_S(h^*)\le \nu$, the robustly-reliable region of $\cL$ is defined as $RR^{\cL}(S,\nu,\eta)=\{x\in\cX\mid r^\cL_{S}(x) \ge \eta\}$ for $\nu,\eta\ge 0$. \red{also define safely-reliable region}
\end{definition}We generalize our results from Section \ref{sec: rr-w-ball} to the agnostic setting (proof details in Appendix \ref{app:agnostic}). Here $\cH_\nu(S)=\{h\in\cH\mid \err_S(h)\le \nu\}$. 

\begin{theorem}\label{thm:rr-eta-l1-agnostic}
Let $\cH$ be any hypothesis class. With respect to $\cM$-ball attacks and 
$\ellca$, for $\eta\ge 0$, 
\begin{enumerate}[label=(\alph*),leftmargin=*,topsep=4pt,partopsep=1ex,parsep=1ex]\itemsep=-4pt
    \item There exists a robustly-reliable learner $\cL$ such that $RR^{\cL}_{\text{CA}}(S,\nu,\eta)\supseteq \operatorname{Agree}(\cH_\nu(S))$,
    \item For any robustly-reliable learner $\cL$, $\rrca(S,\nu,\eta)\subseteq \operatorname{Agree}(\cH_\nu(S))$.
\end{enumerate}

\end{theorem}
}

\blue{
\section{Discussion}
In this work, we generalize the classical line of works on reliable learning to address challenging test-time environments. We propose a novel robustly-reliability criterion that is applicable to several variations of test-time attacks. Our analysis leads to an easy-to-derive algorithm that can be implemented efficiently in many cases. Additionally, we introduce a $\cP \to \cQ$ disagreement coefficient to capture the transferability of the reliability guarantee between distributions. The proposed robustly-reliability criterion and the $\cP \to \cQ$ disagreement coefficient together provide a comprehensive framework for handling test-time attacks and evaluating the reliability of learning models. This contributes to the advancement of reliable learning methodologies in the face of challenging real-world scenarios, facilitating the development of more resilient and trustworthy machine learning systems. \cready{Notably, key questions remain open, including, how to efficiently implement the algorithm for a class of neural networks, and how to learn reliably with respect to any general robust loss function?

\section{Acknowledgements}
This work was supported in part by NSF grants CCF-1910321 and SES-1919453, the Defense Advanced Research Projects Agency under cooperative agreement HR00112020003, a  Bloomberg Data Science PhD fellowship, and a Simons Investigator Award.

}

}

\bibliographystyle{alpha}
\bibliography{reference}

@inproceedings{montasser2019vc,
  title={{VC} classes are adversarially robustly learnable, but only improperly},
  author={Montasser, Omar and Hanneke, Steve and Srebro, Nathan},
  booktitle={Conference on Learning Theory},
  pages={2512--2530},
  year={2019},
  organization={PMLR}
}

@inproceedings{awasthi2014power,
  title={The power of localization for efficiently learning linear separators with noise},
  author={Awasthi, Pranjal and Balcan, Maria Florina and Long, Philip M},
  booktitle={Proceedings of the forty-sixth annual ACM Symposium on Theory of Computing},
  pages={449--458},
  year={2014}
}

@inproceedings{balcan2022robustly,
  title={Robustly-reliable learners under poisoning attacks},
  author={Balcan, Maria-Florina and Blum, Avrim and Hanneke, Steve and Sharma, Dravyansh},
  booktitle={Conference on Learning Theory},
  year={2022},
  organization={PMLR}
}

@article{montasser2020reducing,
  title={Reducing adversarially robust learning to non-robust pac learning},
  author={Montasser, Omar and Hanneke, Steve and Srebro, Nati},
  journal={Advances in Neural Information Processing Systems},
  volume={33},
  pages={14626--14637},
  year={2020}
}

@inproceedings{hopkins2020noise,
  title={Noise-tolerant, reliable active classification with comparison queries},
  author={Hopkins, Max and Kane, Daniel and Lovett, Shachar and Mahajan, Gaurav},
  booktitle={Conference on Learning Theory},
  pages={1957--2006},
  year={2020},
  organization={PMLR}
}

@inproceedings{awasthi2015efficient,
  title={Efficient learning of linear separators under bounded noise},
  author={Awasthi, Pranjal and Balcan, Maria-Florina and Haghtalab, Nika and Urner, Ruth},
  booktitle={Conference on Learning Theory},
  pages={167--190},
  year={2015},
  organization={PMLR}
}

@inproceedings{daniely2016complexity,
  title={Complexity theoretic limitations on learning halfspaces},
  author={Daniely, Amit},
  booktitle={Proceedings of the forty-eighth annual ACM Symposium on Theory of Computing},
  pages={105--117},
  year={2016}
}

@article{diakonikolas2019distribution,
  title={Distribution-independent pac learning of halfspaces with massart noise},
  author={Diakonikolas, Ilias and Gouleakis, Themis and Tzamos, Christos},
  journal={Advances in Neural Information Processing Systems},
  volume={32},
  year={2019}
}

@book{vapnik1998statistical,
  title={Statistical learning theory},
  author={Vapnik, Vladimir N},
  year={1998},
  publisher={Wiley-Interscience}
}

@book{kearns1994introduction,
  title={An introduction to computational learning theory},
  author={Kearns, Michael J and Vazirani, Umesh},
  year={1994},
  publisher={MIT press}
}

@article{kalai2008agnostically,
  title={Agnostically learning halfspaces},
  author={Kalai, Adam Tauman and Klivans, Adam R and Mansour, Yishay and Servedio, Rocco A},
  journal={SIAM Journal on Computing},
  volume={37},
  number={6},
  pages={1777--1805},
  year={2008},
  publisher={SIAM}
}

@inproceedings{shafahiadversarially,
  title={Adversarially robust transfer learning},
  author={Shafahi, Ali and Saadatpanah, Parsa and Zhu, Chen and Ghiasi, Amin and Studer, Christoph and Jacobs, David and Goldstein, Tom},
  booktitle={International Conference on Learning Representations},
  year={2020}
}

@article{deng2023hardness,
  title={On the Hardness of Robustness Transfer: A Perspective from {R}ademacher Complexity over Symmetric Difference Hypothesis Space},
  author={Deng, Yuyang and Gazagnadou, Nidham and Hong, Junyuan and Mahdavi, Mehrdad and Lyu, Lingjuan},
  journal={arXiv preprint arXiv:2302.12351},
  year={2023}
}

@article{salman2020adversarially,
  title={Do adversarially robust imagenet models transfer better?},
  author={Salman, Hadi and Ilyas, Andrew and Engstrom, Logan and Kapoor, Ashish and Madry, Aleksander},
  journal={Advances in Neural Information Processing Systems},
  volume={33},
  pages={3533--3545},
  year={2020}
}

@article{mitchell1982generalization,
  title={Generalization as search},
  author={Mitchell, Tom M},
  journal={Artificial Intelligence},
  volume={18},
  number={2},
  pages={203--226},
  year={1982},
  publisher={Elsevier}
}

@article{balcan2020power,
  title={On the power of abstention and data-driven decision making for adversarial robustness},
  author={Balcan, Maria-Florina and Blum, Avrim and Sharma, Dravyansh and Zhang, Hongyang},
  journal={Journal of Machine Learning Research},
  year={2023}
}

@inproceedings{rivest1988learning,
  title={Learning complicated concepts reliably and usefully},
  author={Rivest, Ronald L and Sloan, Robert H},
  booktitle={Association for the Advancement of Artificial Intelligence},
  pages={635--640},
  year={1988}
}

@inproceedings{cohen2019certified,
  title={Certified adversarial robustness via randomized smoothing},
  author={Cohen, Jeremy and Rosenfeld, Elan and Kolter, Zico},
  booktitle={{International Conference on Machine Learning}},
  pages={1310--1320},
  year={2019},
  organization={PMLR}
}

@article{hanneke2019value,
  title={On the value of target data in transfer learning},
  author={Hanneke, Steve and Kpotufe, Samory},
  journal={Advances in Neural Information Processing Systems},
  volume={32},
  year={2019}
}

@article{mansour2008domain,
  title={Domain adaptation with multiple sources},
  author={Mansour, Yishay and Mohri, Mehryar and Rostamizadeh, Afshin},
  journal={Advances in Neural Information Processing Systems},
  volume={21},
  year={2008}
}

@article{ben2006analysis,
  title={Analysis of representations for domain adaptation},
  author={Ben-David, Shai and Blitzer, John and Crammer, Koby and Pereira, Fernando},
  journal={Advances in Neural Information Processing Systems},
  volume={19},
  year={2006}
}

@inproceedings{recht2019imagenet,
  title={Do imagenet classifiers generalize to imagenet?},
  author={Recht, Benjamin and Roelofs, Rebecca and Schmidt, Ludwig and Shankar, Vaishaal},
  booktitle={International Conference on Machine Learning},
  pages={5389--5400},
  year={2019},
  organization={PMLR}
}

@inproceedings{madry2018towards,
  title={Towards Deep Learning Models Resistant to Adversarial Attacks},
  author={Madry, Aleksander and Makelov, Aleksandar and Schmidt, Ludwig and Tsipras, Dimitris and Vladu, Adrian},
  booktitle={International Conference on Learning Representations},
  year={2018}
}

@inproceedings{carlini2017towards,
  title={Towards evaluating the robustness of neural networks},
  author={Carlini, Nicholas and Wagner, David},
  booktitle={2017 {IEEE} symposium on security and privacy},
  pages={39--57},
  year={2017},
  organization={IEEE}
}

@article{goodfellow2014explaining,
  title={Explaining and harnessing adversarial examples},
  author={Goodfellow, Ian J and Shlens, Jonathon and Szegedy, Christian},
  journal={International Conference on Learning Representations},
  year={2015}
}

@inproceedings{hanneke2007bound,
  title={A bound on the label complexity of agnostic active learning},
  author={Hanneke, Steve},
  booktitle={Proceedings of the 24th {International Conference on Machine Learning}},
  pages={353--360},
  year={2007}
}

@book{van1996weak,
  title={Weak convergence},
  author={van der Vaart, Aad W and Wellner, Jon A},
  year={1996},
  publisher={Springer}
}

@article{wang2011smoothness,
  title={Smoothness, Disagreement Coefficient, and the Label Complexity of Agnostic Active Learning.},
  author={Wang, Liwei},
  journal={Journal of Machine Learning Research},
  volume={12},
  number={7},
  year={2011}
}

@inproceedings{balcan2013active,
  title={Active and passive learning of linear separators under log-concave distributions},
  author={Balcan, Maria-Florina and Long, Phil},
  booktitle={Conference on Learning Theory},
  pages={288--316},
  year={2013},
  organization={PMLR}
}

@article{balcan2017sample,
  title={Sample and computationally efficient learning algorithms under s-concave distributions},
  author={Balcan, Maria-Florina and Zhang, Hongyang},
  journal={Advances in Neural Information Processing Systems},
  volume={30},
  year={2017}
}

@inproceedings{applegate1991sampling,
  title={Sampling and integration of near log-concave functions},
  author={Applegate, David and Kannan, Ravi},
  booktitle={Proceedings of the twenty-third annual ACM Symposium on Theory of Computing},
  pages={156--163},
  year={1991}
}

@book{anthony1999neural,
  title={Neural network learning: Theoretical foundations},
  author={Anthony, Martin and Bartlett, Peter L},
  year={1999},
  publisher={{Cambridge University Press}}
}

@article{steinhardt2017certified,
  title={Certified defenses for data poisoning attacks},
  author={Steinhardt, Jacob and Koh, Pang Wei W and Liang, Percy S},
  journal={Advances in Neural Information Processing Systems},
  volume={30},
  year={2017}
}

@article{cohn1994improving,
  title={Improving generalization with active learning},
  author={Cohn, David and Atlas, Les and Ladner, Richard},
  journal={Machine Learning},
  volume={15},
  pages={201--221},
  year={1994},
  publisher={Springer}
}

@inproceedings{wang2022improved,
  title={Improved certified defenses against data poisoning with (deterministic) finite aggregation},
  author={Wang, Wenxiao and Levine, Alexander J and Feizi, Soheil},
  booktitle={International Conference on Machine Learning},
  pages={22769--22783},
  year={2022},
  organization={PMLR}
}

@inproceedings{lipton2018detecting,
  title={Detecting and correcting for label shift with black box predictors},
  author={Lipton, Zachary and Wang, Yu-Xiang and Smola, Alexander},
  booktitle={International Conference on Machine Learning},
  pages={3122--3130},
  year={2018},
  organization={PMLR}
}

@inproceedings{miller2021accuracy,
  title={Accuracy on the line: on the strong correlation between out-of-distribution and in-distribution generalization},
  author={Miller, John P and Taori, Rohan and Raghunathan, Aditi and Sagawa, Shiori and Koh, Pang Wei and Shankar, Vaishaal and Liang, Percy and Carmon, Yair and Schmidt, Ludwig},
  booktitle={International Conference on Machine Learning},
  pages={7721--7735},
  year={2021},
  organization={PMLR}
}

@article{lovasz2007geometry,
  title={The geometry of logconcave functions and sampling algorithms},
  author={Lov{\'a}sz, L{\'a}szl{\'o} and Vempala, Santosh},
  journal={Random Structures \& Algorithms},
  volume={30},
  number={3},
  pages={307--358},
  year={2007},
  publisher={Wiley Online Library}
}

@inproceedings{attias2019improved,
  title={Improved generalization bounds for robust learning},
  author={Attias, Idan and Kontorovich, Aryeh and Mansour, Yishay},
  booktitle={Algorithmic Learning Theory},
  pages={162--183},
  year={2019},
  organization={PMLR}
}

@article{el2012active,
  title={Active Learning via Perfect Selective Classification.},
  author={El-Yaniv, Ran and Wiener, Yair},
  journal={Journal of Machine Learning Research},
  volume={13},
  number={2},
  year={2012}
}

@inproceedings{gal2016dropout,
  title={Dropout as a bayesian approximation: Representing model uncertainty in deep learning},
  author={Gal, Yarin and Ghahramani, Zoubin},
  booktitle={International Conference on Machine Learning},
  pages={1050--1059},
  year={2016},
  organization={PMLR}
}

@inproceedings{kifer2004detecting,
  title={Detecting change in data streams},
  author={Kifer, Daniel and Ben-David, Shai and Gehrke, Johannes},
  booktitle={VLDB},
  volume={4},
  pages={180--191},
  year={2004},
  organization={Toronto, Canada}
}

@book{williams2006gaussian,
  title={Gaussian processes for machine learning},
  author={Williams, Christopher KI and Rasmussen, Carl Edward},
  year={2006},
  publisher={MIT press Cambridge, MA}
}

@inproceedings{blundell2015weight,
  title={Weight uncertainty in neural network},
  author={Blundell, Charles and Cornebise, Julien and Kavukcuoglu, Koray and Wierstra, Daan},
  booktitle={International Conference on Machine Learning},
  pages={1613--1622},
  year={2015},
  organization={PMLR}
}

@article{lakshminarayanan2017simple,
  title={Simple and scalable predictive uncertainty estimation using deep ensembles},
  author={Lakshminarayanan, Balaji and Pritzel, Alexander and Blundell, Charles},
  journal={Advances in Neural Information Processing Systems},
  volume={30},
  year={2017}
}

@article{maddox2019simple,
  title={A simple baseline for bayesian uncertainty in deep learning},
  author={Maddox, Wesley J and Izmailov, Pavel and Garipov, Timur and Vetrov, Dmitry P and Wilson, Andrew Gordon},
  journal={Advances in Neural Information Processing Systems},
  volume={32},
  year={2019}
}

@inproceedings{lecuyer2019certified,
  title={Certified robustness to adversarial examples with differential privacy},
  author={Lecuyer, Mathias and Atlidakis, Vaggelis and Geambasu, Roxana and Hsu, Daniel and Jana, Suman},
  booktitle={2019 IEEE Symposium on Security and Privacy},
  pages={656--672},
  year={2019},
  organization={IEEE}
}

@article{li2019certified,
  title={Certified adversarial robustness with additive noise},
  author={Li, Bai and Chen, Changyou and Wang, Wenlin and Carin, Lawrence},
  journal={Advances in Neural Information Processing Systems},
  volume={32},
  year={2019}
}

@book{quinonero2008dataset,
  title={Dataset shift in machine learning},
  author={Quinonero-Candela, Joaquin and Sugiyama, Masashi and Schwaighofer, Anton and Lawrence, Neil D},
  year={2008},
  publisher={MIT Press}
}

@inproceedings{sun2016deep,
  title={Deep coral: Correlation alignment for deep domain adaptation},
  author={Sun, Baochen and Saenko, Kate},
  booktitle={Computer Vision--ECCV 2016 Workshops: Amsterdam, The Netherlands, October 8-10 and 15-16, 2016, Proceedings, Part III 14},
  pages={443--450},
  year={2016},
  organization={Springer}
}

@article{arjovsky2019invariant,
  title={Invariant risk minimization},
  author={Arjovsky, Martin and Bottou, L{\'e}on and Gulrajani, Ishaan and Lopez-Paz, David},
  journal={arXiv preprint arXiv:1907.02893},
  year={2019}
}

@inproceedings{sagawadistributionally,
  title={Distributionally Robust Neural Networks},
  author={Sagawa, Shiori and Koh, Pang Wei and Hashimoto, Tatsunori B and Liang, Percy},
  booktitle={International Conference on Learning Representations},
  year={2020}
}

@article{cao2019learning,
  title={Learning imbalanced datasets with label-distribution-aware margin loss},
  author={Cao, Kaidi and Wei, Colin and Gaidon, Adrien and Arechiga, Nikos and Ma, Tengyu},
  journal={Advances in Neural Information Processing Systems},
  volume={32},
  year={2019}
}

@article{ben2010theory,
  title={A theory of learning from different domains},
  author={Ben-David, Shai and Blitzer, John and Crammer, Koby and Kulesza, Alex and Pereira, Fernando and Vaughan, Jennifer Wortman},
  journal={Machine Learning},
  volume={79},
  pages={151--175},
  year={2010},
  publisher={Springer}
}

@inproceedings{balcan2022nash,
  title={Nash equilibria and pitfalls of adversarial training in adversarial robustness games},
  author={Balcan, Maria-Florina and Pukdee, Rattana and Ravikumar, Pradeep and Zhang, Hongyang},
  booktitle={The 26nd International Conference on Artificial Intelligence and Statistics},
  year={2023},
  organization={PMLR}
}

@inproceedings{javanmard2020precise,
  title={Precise tradeoffs in adversarial training for linear regression},
  author={Javanmard, Adel and Soltanolkotabi, Mahdi and Hassani, Hamed},
  booktitle={Conference on Learning Theory},
  pages={2034--2078},
  year={2020},
  organization={PMLR}
}

@inproceedings{raghunathan2020understanding,
  title={Understanding and mitigating the tradeoff between robustness and accuracy},
  author={Raghunathan, Aditi and Xie, Sang Michael and Yang, Fanny and Duchi, John C and Liang, Percy},
  booktitle={Proceedings of the 37th International Conference on Machine Learning},
  pages={7909--7919},
  year={2020}
}

@inproceedings{puchkin2021exponential,
  title={Exponential savings in agnostic active learning through abstention},
  author={Puchkin, Nikita and Zhivotovskiy, Nikita},
  booktitle={Conference on Learning Theory},
  pages={3806--3832},
  year={2021},
  organization={PMLR}
}

@inproceedings{zhuefficient,
  title={Efficient Active Learning with Abstention},
  author={Zhu, Yinglun and Nowak, Robert D},
  booktitle={Advances in Neural Information Processing Systems},
  year={2022}
}

@inproceedings{cortes2018online,
  title={Online learning with abstention},
  author={Cortes, Corinna and DeSalvo, Giulia and Gentile, Claudio and Mohri, Mehryar and Yang, Scott},
  booktitle={International Conference on Machine Learning},
  pages={1059--1067},
  year={2018},
  organization={PMLR}
}

@article{yan2016active,
  title={Active learning from imperfect labelers},
  author={Yan, Songbai and Chaudhuri, Kamalika and Javidi, Tara},
  journal={Advances in Neural Information Processing Systems},
  volume={29},
  year={2016}
}

@article{cortes2016boosting,
  title={Boosting with abstention},
  author={Cortes, Corinna and DeSalvo, Giulia and Mohri, Mehryar},
  journal={Advances in Neural Information Processing Systems},
  volume={29},
  year={2016}
}

@inproceedings{guo2017calibration,
  title={On calibration of modern neural networks},
  author={Guo, Chuan and Pleiss, Geoff and Sun, Yu and Weinberger, Kilian Q},
  booktitle={International Conference on Machine Learning},
  pages={1321--1330},
  year={2017},
  organization={PMLR}
}

@inproceedings{tsiprasrobustness,
  title={Robustness May Be at Odds with Accuracy},
  author={Tsipras, Dimitris and Santurkar, Shibani and Engstrom, Logan and Turner, Alexander and Madry, Aleksander},
  booktitle={International Conference on Learning Representations},
  year={2019}
}

@article{schmidt2018adversarially,
  title={Adversarially robust generalization requires more data},
  author={Schmidt, Ludwig and Santurkar, Shibani and Tsipras, Dimitris and Talwar, Kunal and Madry, Aleksander},
  journal={Advances in Neural Information Processing Systems},
  volume={31},
  year={2018}
}

@inproceedings{szegedy2013intriguing,
title	= {Intriguing properties of neural networks},
author	= {Christian Szegedy and Wojciech Zaremba and Ilya Sutskever and Joan Bruna and Dumitru Erhan and Ian Goodfellow and Rob Fergus},
year	= {2014},
booktitle	= {International Conference on Learning Representations}
}

@inproceedings{montasser2021adversarially,
  title={Adversarially robust learning with unknown perturbation sets},
  author={Montasser, Omar and Hanneke, Steve and Srebro, Nathan},
  booktitle={Conference on Learning Theory},
  pages={3452--3482},
  year={2021},
  organization={PMLR}
}

@inproceedings{zhang2019theoretically,
  title={Theoretically principled trade-off between robustness and accuracy},
  author={Zhang, Hongyang and Yu, Yaodong and Jiao, Jiantao and Xing, Eric and El Ghaoui, Laurent and Jordan, Michael},
  booktitle={International Conference on Machine Learning},
  pages={7472--7482},
  year={2019},
  organization={PMLR}
}

@article{shafahi2019adversarial,
  title={Adversarial training for free!},
  author={Shafahi, Ali and Najibi, Mahyar and Ghiasi, Mohammad Amin and Xu, Zheng and Dickerson, John and Studer, Christoph and Davis, Larry S and Taylor, Gavin and Goldstein, Tom},
  journal={Advances in Neural Information Processing Systems},
  volume={32},
  year={2019}
}

@inproceedings{rice2020overfitting,
  title={Overfitting in adversarially robust deep learning},
  author={Rice, Leslie and Wong, Eric and Kolter, Zico},
  booktitle={International Conference on Machine Learning},
  pages={8093--8104},
  year={2020},
  organization={PMLR}
}

@inproceedings{tramer2018ensemble,
  title={Ensemble Adversarial Training: Attacks and Defenses},
  author={Tram{\`e}r, Florian and Kurakin, Alexey and Papernot, Nicolas and Goodfellow, Ian and Boneh, Dan and McDaniel, Patrick},
  booktitle={International Conference on Learning Representations},
  year={2018}
}

@article{carmon2019unlabeled,
  title={Unlabeled data improves adversarial robustness},
  author={Carmon, Yair and Raghunathan, Aditi and Schmidt, Ludwig and Duchi, John C and Liang, Percy S},
  journal={Advances in Neural Information Processing Systems},
  volume={32},
  year={2019}
}

@article{goldwasser2020beyond,
  title={Beyond perturbations: Learning guarantees with arbitrary adversarial test examples},
  author={Goldwasser, Shafi and Kalai, Adam Tauman and Kalai, Yael and Montasser, Omar},
  journal={Advances in Neural Information Processing Systems},
  volume={33},
  pages={15859--15870},
  year={2020}
}

@inproceedings{bickel2007discriminative,
  title={Discriminative learning for differing training and test distributions},
  author={Bickel, Steffen and Br{\"u}ckner, Michael and Scheffer, Tobias},
  booktitle={Proceedings of the 24th {International Conference on Machine Learning}},
  pages={81--88},
  year={2007}
}

@article{namkoong2016stochastic,
  title={Stochastic gradient methods for distributionally robust optimization with f-divergences},
  author={Namkoong, Hongseok and Duchi, John C},
  journal={Advances in Neural Information Processing Systems},
  volume={29},
  year={2016}
}

@article{duchi2021learning,
  title={Learning models with uniform performance via distributionally robust optimization},
  author={Duchi, John C and Namkoong, Hongseok},
  journal={The Annals of Statistics},
  volume={49},
  number={3},
  pages={1378--1406},
  year={2021},
  publisher={Institute of Mathematical Statistics}
}

@article{pan2010survey,
  title={A survey on transfer learning},
  author={Pan, Sinno Jialin and Yang, Qiang},
  journal={IEEE Transactions on knowledge and data engineering},
  volume={22},
  number={10},
  pages={1345--1359},
  year={2010},
  publisher={IEEE}
}

@incollection{torrey2010transfer,
  title={Transfer learning},
  author={Torrey, Lisa and Shavlik, Jude},
  booktitle={Handbook of research on machine learning applications and trends: algorithms, methods, and techniques},
  pages={242--264},
  year={2010},
  publisher={IGI global}
}

@article{shimodaira2000improving,
  title={Improving predictive inference under covariate shift by weighting the log-likelihood function},
  author={Shimodaira, Hidetoshi},
  journal={Journal of Statistical Planning and Inference},
  volume={90},
  number={2},
  pages={227--244},
  year={2000},
  publisher={Elsevier}
}

@article{huang2006correcting,
  title={Correcting sample selection bias by unlabeled data},
  author={Huang, Jiayuan and Gretton, Arthur and Borgwardt, Karsten and Sch{\"o}lkopf, Bernhard and Smola, Alex},
  journal={Advances in Neural Information Processing Systems},
  volume={19},
  year={2006}
}

@inproceedings{rosenfeldrisks,
  title={The Risks of Invariant Risk Minimization},
  author={Rosenfeld, Elan and Ravikumar, Pradeep Kumar and Risteski, Andrej},
  booktitle={International Conference on Learning Representations},
  year={2021}
}

@inproceedings{zhang2019bridging,
  title={Bridging theory and algorithm for domain adaptation},
  author={Zhang, Yuchen and Liu, Tianle and Long, Mingsheng and Jordan, Michael},
  booktitle={International Conference on Machine Learning},
  pages={7404--7413},
  year={2019},
  organization={PMLR}
}

@article{el2010foundations,
  title={On the Foundations of Noise-free Selective Classification.},
  author={El-Yaniv, Ran and  Wiener, Yair},
  journal={Journal of Machine Learning Research},
  volume={11},
  number={5},
  year={2010}
}

@inproceedings{zhang2019defending,
  title={Defending against whitebox adversarial attacks via randomized discretization},
  author={Zhang, Yuchen and Liang, Percy},
  booktitle={The 22nd International Conference on Artificial Intelligence and Statistics},
  pages={684--693},
  year={2019},
  organization={PMLR}
}

@article{montasseradversarially,
  title={Adversarially Robust Learning: A Generic Minimax Optimal Learner and Characterization},
  author={Montasser, Omar and Hanneke, Steve and Srebro, Nathan},
  journal={Advances in Neural Information Processing Systems},
  year = {2022}
}

@inproceedings{balcan2006agnostic,
  title={Agnostic active learning},
  author={Balcan, Maria-Florina and Beygelzimer, Alina and Langford, John},
  booktitle={Proceedings of the 23rd International Conference on Machine Learning},
  pages={65--72},
  year={2006}
}

@article{zhang2020localized,
  title={On localized discrepancy for domain adaptation},
  author={Zhang, Yuchen and Long, Mingsheng and Wang, Jianmin and Jordan, Michael I},
  journal={arXiv preprint arXiv:2008.06242},
  year={2020}
}

@article{balcan2021noise,
  title={Noise in Classification},
  author={Balcan, Maria-Florina and Haghtalab, Nika},
  journal={Beyond the Worst-Case Analysis of Algorithms},
  pages={361},
  year={2021},
  publisher={Cambridge University Press}
}

@inproceedings{zhang2021incentive,
  title={Incentive-aware PAC learning},
  author={Zhang, Hanrui and Conitzer, Vincent},
  booktitle={Proceedings of the AAAI Conference on Artificial Intelligence},
  volume={35},
  number={6},
  pages={5797--5804},
  year={2021}
}

@article{gourdeau2021hardness,
  title={On the hardness of robust classification},
  author={Gourdeau, Pascale and Kanade, Varun and Kwiatkowska, Marta and Worrell, James},
  journal={The Journal of Machine Learning Research},
  volume={22},
  number={1},
  pages={12521--12549},
  year={2021},
  publisher={JMLRORG}
}

\appendix

\clearpage
\section{Additional related work}
\label{appendix: related work}

\textbf{Reliability. } A learning model that outputs a confidence level is valuable in practical applications, as it allows us to determine when to trust the model and when to defer the task to a human. However, it is well-known that models like neural networks can exhibit high confidence, yet still produce incorrect results \cite{guo2017calibration}. To tackle this issue, there has been a line of works on learning algorithms with uncertainty estimate \cite{williams2006gaussian, blundell2015weight, gal2016dropout,lakshminarayanan2017simple,maddox2019simple}.  Unlike prior work, our results take into account the relevant notion of robust loss. In particular, we extend the reliability guarantees in
perfect selective classification \cite{el2012active} and reliable-useful learning model \cite{rivest1988learning} to different robust losses under a test-time attack. Prior work on reliability under data poisoning attacks \cite{balcan2022robustly} obtained similar results on training-time attacks, by providing guarantees that the learner is always correct at any point that it makes a prediction provided the training data corruption does not exceed a point-specific threshold. Our work is also related to learning algorithms with an abstention option \cite{yan2016active,cortes2016boosting, cortes2018online,puchkin2021exponential,zhuefficient}.

\noindent \textbf{Robustness. } Robustness against adversarial attacks is essential for the safe deployment of machine learning models in the real world. Our focus in this work is on perturbation attacks, where we aim to provide learners that remain robust even when the test data points are perturbed. It is known that many modern approaches such as deep neural networks fail in this case even when the perturbation is human-imperceptible \cite{szegedy2013intriguing, goodfellow2014explaining}. There has been a lot of empirical effort \cite{madry2018towards,zhang2019theoretically, shafahi2019adversarial, rice2020overfitting,tramer2018ensemble,carmon2019unlabeled} as well as theoretical effort \cite{tsiprasrobustness,schmidt2018adversarially,javanmard2020precise, raghunathan2020understanding,montasser2019vc,montasser2021adversarially,goldwasser2020beyond, balcan2022nash} to develop learners with improved robustness, and more broadly to understand various aspects of adversarial robustness. In particular, there is a line of work on certified robustness \cite{cohen2019certified,lecuyer2019certified,li2019certified} which provides a pointwise guarantee that the prediction does not change, so long as the attack strength is within a learner-specified `radius' for the point. While the certified robustness research focuses on this consistency aspect, our work addresses the reliability aspect where we hope to guarantee that the prediction is also correct.

\noindent \textbf{Distribution shift. }
A distribution shift refers to the phenomenon where the training distribution differs from the test distribution which often leads to a degradation in the learner's performance. This has been studied under several different settings \cite{quinonero2008dataset}, ranging from covariate shift \cite{shimodaira2000improving,huang2006correcting,bickel2007discriminative}, and domain adaptation \cite{mansour2008domain, ben2010theory, zhang2019bridging, zhang2020localized} to transfer learning \cite{pan2010survey,torrey2010transfer, hanneke2019value}. Many algorithms have been proposed to deal with the shift which involves encouraging invariance between different domains \cite{sun2016deep,arjovsky2019invariant,rosenfeldrisks}  or taking into account the worst-case subpopulation \cite{namkoong2016stochastic, duchi2021learning, sagawadistributionally, cao2019learning}. While prior work typically focuses on the average performance on the target domain or subpopulation, we provide point-wise reliability guarantees.

\noindent \textbf{Learning with noise. } There is extensive classic literature on learning methods which are tolerant or robust to noise \cite{kearns1994introduction,vapnik1998statistical}---including efficient learning under bounded or Massart noise \cite{awasthi2015efficient,diakonikolas2019distribution}, agnostic active learning \cite{kalai2008agnostically,daniely2016complexity}, learning under malicious noise \cite{awasthi2014power,balcan2021noise}, to list a few. Recent work has considered reliable learning under some of these classic noise models \cite{hopkins2020noise,balcan2022robustly}.

\section{Additional proof details for robustly-reliable learners w.r.t. metric ball attacks}\label{appendix: robustly-reliable metric ball}
\begin{theorem}\label{thm:rr-eta-l1}
Let $\cH$ be any hypothesis class. With respect to $\cM$-ball attacks and $\ellca$, for $\eta\ge 0$, 
\begin{enumerate}[label=(\alph*),leftmargin=*,topsep=4pt,partopsep=1ex,parsep=1ex]\itemsep=-4pt
    \item there exists a robustly-reliable learner $\cL$ such that $\rrca(S,\eta)\supseteq \operatorname{Agree}(\cH_0(S))$, and
    \item for any robustly-reliable learner $\cL$, $\rrca(S,\eta)\subseteq \operatorname{Agree}(\cH_0(S))$.
\end{enumerate}
The results hold for $\rrtl$ as well.
\end{theorem}
\begin{proof}
(Proof of Theorem \ref{thm:rr-eta-l1})
The robustly-reliable learner $\cL$ is given as follows. 
 Set $h^{\cL}_S=\argmin_{h\in\cH}\operatorname{err}_S(h)$ i.e. an ERM over $S$, and $r^{\cL}_S(z)=\infty$ if $z\in\operatorname{Agree}(\cH_0(S))$, else $r^{\cL}_S(z)=-1$.
 By realizability, $\operatorname{err}_S(h^{\cL}_S)\le \operatorname{err}_S(h^*)=0$, or $h^{\cL}_S\in\cH_0(S)$.  We first show that $\cL$ is robustly-reliable. For $z\in\cX$, if $r^{\cL}_S(z)=\eta> 0$, then $z\in \operatorname{Agree}(\cH_0(S))$. 
 We have $h^*(z)=h^{\cL}_S(z)$ since the classifiers $h^*,h^{\cL}_S\in \cH_0(S)$ and $z$ lies in the agreement region of classifiers in $\cH_0(S)$ in this case. Thus, we have $\ellca^{h^*}(h^{\cL}_S,x,z)=0$ for any $x$ such that $z\in\B_\cM^o(x,\eta)$. The $\eta=0$ case corresponds to reliability in the absence of test-time attack, so \cite{el2010foundations} applies. Therefore, $\rrca(S,\eta)\supseteq \operatorname{Agree}(\cH_0(S))$ for all $\eta\ge 0$ follows from the setting $r^{\cL}_S(z)=\infty$ if $z\in\operatorname{Agree}(\cH_0(S))$.

Conversely, let $z\in \operatorname{DIS}(\cH_0(S))$. There exist $h_1,h_2\in \cH_0(S)$ such that $h_1(z)\ne h_2(z)$. If possible,  let there be a robustly-reliable learner $\cL$ such that $z\in \rrca(S,\eta)$ for some $\eta>0$. By definition of the robust-reliability region, we must have $r^{\cL}_S(z)> 0$. By definition of a  ball, we have $z\in\B_\cM^o(z,\eta)$ for any $\eta>0$, and therefore $\ellca^{h^*}(h^{\cL}_S,z,z)=0$. But then we must have $h^{\cL}_S(z)=h^*(z)$ by definition of $\ellca$. But we can set $h^*=h_1$ or $h^*=h_2$ since both are consistent with $S$. But $h_1(z)\ne h_2(z)$, and therefore $h^{\cL}_S(z)\ne h^*(z)$ for one of the above choices for $h^*$, contradicting that $\cL$ is robustly-reliable.    
\end{proof}

\begin{theorem}\label{thm:mball-l3-ub}
Let $\cH$ be any hypothesis class. With respect to $\cM$-ball attacks and $\ellst$, for $\eta \geq 0$,
\begin{enumerate} [label=(\alph*),leftmargin=*,topsep=4pt,partopsep=1ex,parsep=1ex]\itemsep=-4pt
    \item there exists a robustly reliable learner $\cL$ such that $\rrst(S,\eta)\supseteq A_{\text{ST}} $, and
    \item for any robustly-reliable learner $\cL$, $\rrst(S,\eta)\subseteq A_{\text{ST}} $,
\end{enumerate}
where \blue{$A_{\text{ST}} = \{z \mid \B^o(z,\eta) \subseteq \operatorname{Agree}(\cH_0(S)) \wedge \forall h \in \cH_0(S), h(x)=h(z),\forall x \in \B^o(z,\eta) \}$} \edit{ for $\eta > 0$ and  $A_{\text{ST}} = \operatorname{Agree}(\cH_0(S))$ for $\eta = 0$}.

\end{theorem}

\begin{proof}
(Proof of Theorem \ref{thm:mball-l3-ub})
Given sample $S$, consider the learner $\cL$ which outputs $h^{\cL}_S=\argmin_{h\in\cH}\operatorname{err}_S(h)$, and $r^{\cL}_S(z)$ is given by the largest $\eta> 0$ for which $\B^o(z,\eta)\subseteq\operatorname{Agree}(\cH_0(S)) $ and $ h(x)=h(z), \;\forall\;x\in \B^o(z,\eta), h\in \cH_0(S)$, else $\eta=0$ if $z\in \operatorname{Agree}(\cH_0(S))$, and $-1$ otherwise. Note that the supremum exists here since a union of open sets is also open. By realizability, $\operatorname{err}_S(h^{\cL}_S)\le \operatorname{err}_S(h^*)=0$, or $h^{\cL}_S\in\cH_0(S)$. We first show that $\cL$ is robustly-reliable w.r.t. $\cM$ for loss $\ellst$. For $z\in\cX$, if $r^{\cL}_S(z)=\eta\ge 0$, then $\B^o(z,\eta)\subseteq\operatorname{Agree}(\cH_0(S))$, in particular $z\in \operatorname{Agree}(\cH_0(S))$. 
Moreover, by definition, for any $x\in \B^o(z,\eta)$, we have $h^{\cL}_S(z)=h^{\cL}_S(x)$ by construction. Putting together, and using the property that distance functions of a metric are symmetric, we have $h^{\cL}_S(z)=h^*(x)$ for any $x$ such that $z\in\B^o(x,\eta)$. Thus, we have $\ellst^{h^*}(h^{\cL}_S,x,z)=0$ for any $x$ such that $z\in\B^o(x,\eta)$. Thus $\cL$ satisfies Definition \ref{def:robustly-reliable-metric}.

Conversely, we will show that for any robustly-reliable learner $\cL$ w.r.t. $\ellst$, for any $\eta>0$,
\begin{equation*}
    \rrst(S,\eta) \subseteq \operatorname{Agree}(\cH_0(S)),
\end{equation*}
which follows from similar arguments from Theorem \ref{thm:rr-eta-l1} which also apply to the $\ellst$ loss. Let $z\in \operatorname{DIS}(H_0(S))$. There exists $h_1,h_2\in H_0(S)$ such that $h_1(z)\ne h_2(z)$. If possible,  let there be a robustly-reliable learner $\cL$ such that $z\in \rrst(S,\eta)$ for some $\eta>0$. By definition of the robust-reliability region, we must have $r^{\cL}_S(z)> 0$. By definition of a closed ball, we have $z\in\B^o_\cM(z,\eta)$ for any $\eta>0$, and therefore $\ellst^{h^*}(h^{\cL}_S,z,z)= \bbI[h_S^{\cL}(z) \neq h^*(z)] = 0$ which implies that $h^{\cL}_S(z)=h^*(z)$. But we can set $h^*=h_1$ or $h^*=h_2$ since both are consistent with $S$. But $h_1(z)\ne h_2(z)$, and therefore $h^{\cL}_S(z)\ne h^*(z)$ for one of the above choices for $h^*$, contradicting that $\cL$ is robustly-reliable. Next, we will show that, for any $\eta>0$,
\begin{equation*}
    \rrst(S,\eta)\subseteq  \{z\mid\B_\cM^o(z,\eta)\subseteq\operatorname{Agree}(\cH_0(S)) \}.
\end{equation*}
We will prove this by contradiction. Suppose $z\in \operatorname{Agree}(H_0(S))$, but there exists $x' \in \B_\cM^o(z,\eta)$ such that $x' \not\in \operatorname{Agree}(H_0(S))$. Let there be a robustly-reliable learner $\cL$ such that $z\in \rrst(S,\eta)$. By definition, we have
$\ellst^{h^*}(h^{\cL}_S,x,z)=0$ for any $x$ that $z \in \B_\cM^o(x,\eta)$. This implies that $\ellst^{h^*}(h^{\cL}_S,x',z)=0$ that is $h^{\cL}_S(z) = h^*(x')$. Because $x' \not\in \operatorname{Agree}(H_0(S))$, there exists $h_1,h_2\in H_0(S)$ such that $h_1(x')\ne h_2(x')$. We can set $h^* = h_1$ or $h^* = h_2$ since both are consistent with $S$. But $h_1(x')\ne h_2(x')$, and therefore $h^{\cL}_S(z)\ne h^*(z)$ for one of the above choices for $h^*$, contradicting that $\cL$ is robustly-reliable. Finally, we will show that, for any $\eta>0$,
\begin{equation*}
\rrst(S,\eta)\subseteq  \{z\mid\B_\cM^o(z,\eta)\subseteq\operatorname{Agree}(\cH_0(S)) \wedge h(x)=h(z), \;\forall\;x\in \B_\cM^o(z,\eta), h\in \cH_0(S)\}.
\end{equation*}
Let $z$ be a data point such that $\B_\cM^o(z,\eta) \subseteq \operatorname{Agree}(H_0(S))$ but there exists $x' \in \B_\cM^o(z,\eta)$ such that $h(x') \neq h(z)$ for $h \in H_0(S)$. Let there be a robustly-reliable learner such that $z \in \rrst(S,\eta)$. This implies that $\ellst(h_S^{\cL},x,z) = 0$ for any $x$ that $z \in \B_\cM^o(z,\eta)$. However, $\ellst(h_S^{\cL},x',z) = \bbI[h_S^{\cL}(z) \neq h^*(x')] = \bbI[h_S^{\cL}(z) \neq h_S^{\cL}(x')] \neq 0$, contradicting that $\cL$ is robustly-reliable.
\end{proof}

\begin{theorem}\label{thm:mball-l4-ub}
Let $\cH$ be any hypothesis class. With respect to $\cM$-ball attacks and $\ellia$, for $\eta \geq 0$,
\begin{enumerate} [label=(\alph*),leftmargin=*,topsep=4pt,partopsep=1ex,parsep=1ex]\itemsep=-4pt
    \item there exists a robustly reliable learner $\cL$ such that $\rria(S,\eta)\supseteq A_{\text{IA}} $, and
    \item for any robustly-reliable learner $\cL$, $\rria(S,\eta)\subseteq A_{\text{IA}} $,
\end{enumerate}
where \blue{$A_{\text{IA}}  =  (A_{\text{ST}} \cap \{z \mid h^*(z) = 1\})\cup \{ z \mid z \in \operatorname{Agree}(\cH_0(S)) \wedge h^*(z) = 0\}$.}
\end{theorem}

\blue{
\begin{proof}
(Proof of Theorem \ref{thm:mball-l4-ub}) The construction of the robustly-reliable learner for $\ellia$ is similar to the robustly-reliable learner for $\ellst$. The key difference is that the reliability radius now depends on the predicted label. Given sample $S$, consider the learner $\cL$ which outputs $h^{\cL}_S=\argmin_{h\in\cH}\operatorname{err}_S(h)$.  
\begin{enumerate}
\item If $h^{\cL}_S(z) = 1$, $r^{\cL}_S(z)$ is given by the largest $\eta> 0$ for which $\B^o(z,\eta)\subseteq\operatorname{Agree}(\cH_0(S)) $ and $ h(x)=h(z), \;\forall\;x\in \B^o(z,\eta), h\in \cH_0(S)$ and $\eta=0$ when $z\in \operatorname{Agree}(\cH_0(S))$, and $-1$ otherwise. Note that the supremum exists here since a union of open sets is also open.
\item If $h^{\cL}_S(z) = 0$, $r^{\cL}_S(z) = \infty$ when $z\in \operatorname{Agree}(\cH_0(S))$, and $-1$ otherwise. 
\end{enumerate}

 We first show that $\cL$ is robustly-reliable w.r.t. $\cM$ for loss $\ellia$. By realizability, $\operatorname{err}_S(h^{\cL}_S)\le \operatorname{err}_S(h^*)=0$, or $h^{\cL}_S\in\cH_0(S)$. For $z\in\cX$, if $h^{\cL}_S(z) = 1$ and $r^{\cL}_S(z)=\eta\ge 0$, then $\B^o(z,\eta)\subseteq\operatorname{Agree}(\cH_0(S))$, in particular $z\in \operatorname{Agree}(\cH_0(S))$. 
Moreover, by definition, for any $x\in \B^o(z,\eta)$, we have $h^{\cL}_S(z)=h^{\cL}_S(x)$ by construction. Putting together, and using the property that distance functions of a metric are symmetric, we have $h^{\cL}_S(z)=h^*(x)$ for any $x$ such that $z\in\B^o(x,\eta)$. Thus, we have $\ellst^{h^*}(h^{\cL}_S,x,z)=0$ for any $x$ such that $z\in\B^o(x,\eta)$. This also implies that $\ellia^{h^*}(h^{\cL}_S,x,z)=0$ since $\ellst$ implies $\ellia$.

On the other hand, if $h^{\cL}_S(z) = 0$ and  $r^{\cL}_S(z)= \infty$, we have $z\in \operatorname{Agree}(\cH_0(S))$. This implies that $h^{\cL}_S(z)=h^*(z)= 0.$ For any $x$ that $z \in \cU_{\text{IA}}(x,h^*)$, by the incentive-aware property of the adversary, if $h^*(x) = 1$, we must have $\cU_{\text{IA}}(x,h^*) = \{x\}$ which implies that $z = x$ and $h^*(z) = h^*(x) = 1$. In our case, $h^*(z)= 0$ also implies that we must also have $h^*(x) = 0$. Therefore, $\ellia^{h^*}(h^{\cL}_S,x,z)= \bbI[h^{\cL}_S(z) \neq h^*(x) \wedge z \in \cU_{\text{IA}}(x,h^*)] \leq \bbI[h^{\cL}_S(z) \neq h^*(x)] = 0.$ Therefore, we can conclude that $\cL$ satisfies Definition \ref{def:robustly-reliable-metric}.

Conversely, we will show that for any robustly-reliable learner $\cL$ w.r.t. $\ellia$, for any $\eta>0$,
\begin{equation*}
    \rria(S,\eta)\cap\{z \mid h^{\cL}_S(z) = 0 \} \subseteq \operatorname{Agree}(\cH_0(S)) \cap\{z \mid h^*(z) = 0 \}.
\end{equation*}
Since $z \in \cU_\text{IA}(z)$, for $z$ to lie in the robustly-reliable region, we need $\ellia^{h^*}(h^{\cL}_S,z,z) = \bbI[h^{\cL}_S(z) \neq h^*(z)] = 0$ that is $z$ must be reliable. By similar arguments from Theorem \ref{thm:rr-eta-l1},  we have the result.

Next, we will show that, for any $\eta>0$,
\begin{equation*}
    \rria(S,\eta) \cap \{z \mid h^{\cL}_S(z) = 1 \}\subseteq  \{z\mid\B_\cM^o(z,\eta)\subseteq\operatorname{Agree}(\cH_0(S)) \}.
\end{equation*}
We will prove this by contradiction. Suppose $z\in \operatorname{Agree}(H_0(S))$, but there exists $x' \in \B_\cM^o(z,\eta)$ such that $x' \not\in \operatorname{Agree}(H_0(S))$. Let there be a robustly-reliable learner $\cL$ such that $z\in \rria(S,\eta)$ and $h^{\cL}_S(z) = 1$. Because $x' \not\in \operatorname{Agree}(H_0(S))$, there exists $h_1\in H_0(S)$ such that $h_1(x')=0$. We may have $h^* = h_1$ since $h_1$ is consistent with $S$. However, we have $\cU_{\text{IA}}(x', h_1) = \B_\cM^o(x',\eta)$ and 
$$\ellia^{h^*}(h^{\cL}_S,x',z)= \bbI[h(z) \neq h^*(x') \wedge z \in \cU_{\text{IA}}(x')] = 1$$
which contradicts with $z$ lies in the robustly-reliable region. Furthermore, with a similar argument that we can't have $h^*(x') = 0$, we can show that the agreed label of any $x$ must be $1$,
\begin{align*}
    &\rria(S,\eta)\cap \{z \mid h^{\cL}_S(z) = 1 \}\\
&\subseteq  \{z\mid\B_\cM^o(z,\eta) \subseteq\operatorname{Agree}(\cH_0(S)) \wedge h(x)= 1, \;\forall\;x\in \B_\cM^o(z,\eta), h\in \cH_0(S)\}.\\
&= A_\text{ST} \cap \{z \mid h^*(z) = 1 \}
\end{align*}
This concludes that for any robustly-reliable learner $\cL$ with respect to $\ellia$, we have
$$
\rria(S,\eta) \subseteq (A_{\text{ST}} \cap \{z \mid h^*(z) = 1\})\cup \{ z \mid z \in \operatorname{Agree}(\cH_0(S)) \wedge h^*(z) = 0\}.
$$
\end{proof}
}

\section{General robustly-reliable learner}\label{appendix: general robustly-reliable}
\begin{definition}[General robustly-reliable learner] \label{def:robustly-reliable}
A learner $\cL$ is robustly-reliable for sample $S$ w.r.t.\ a perturbation function $\cU$, concept space $\cH$ and robust loss function $\ell$ if, for any target concept $h^*\in\cH$, given $S$ labeled by $h^*$, the learner outputs functions $h^\cL_{S}:\cX \to \cY$ and $a^\cL_{S}:\cX \to \{0,1\}$ such that for all $z\in \cX$ if $a^\cL_{S}(z) = 1$ and $z\in\cU(x)$ then $\ell^{h^*}(h^\cL_{S},x,z)=0$. On the other hand, if $a^\cL_{S}(z) = 0$, our learner abstains from prediction. The robustly-reliable region of a learner $\cL$ is defined as $\text{RR}^{\cL}(S)=\{x\in\cX\mid a^\cL_{S}(x) = 1\}$, the region that the learner $\cL$ does not abstain. 
\end{definition}

\noindent We again obtain the pointwise optimal characterization of the robustly-reliable region in terms of the agreement region.  For $\ellca, \elltl$  the robustly-reliable region would be the same as the {region where we can be sure of what the correct label is: i.e.\ the agreement region of the version space} while for {$\ellst$, it is the region of points $z$ for which $\cU^{-1}(z)$ lies inside the agreement region of the version space, and all classifiers in the version space agree on $\cU^{-1}(z)$}.

\begin{theorem}
\label{thm: rrr-l3-lb-general}
Let $\cH$ be any hypothesis class, and $\cU$ be the perturbation function.

\begin{enumerate}[label=(\alph*),leftmargin=*,topsep=0pt,partopsep=1ex,parsep=1ex]\itemsep=-4pt   
    \item There exists a robustly-reliable learner $\cL$ w.r.t.\ $\cU$ and $\ellca$ such that $\rrca(S)\supseteq \operatorname{Agree}(\cH_0(S))$. Moreover, for any robustly-reliable learner $\cL$, $\rrca(S)\subseteq \operatorname{Agree}(\cH_0(S))$.
    \item The same results hold for $\rrtl$ as well.
    \item There exists a robustly-reliable learner $\cL$ w.r.t.\ $\cU$ and $\ellst$, such that $\rrst(S)\supseteq A_{\text{ST}} $, and for any $\cL$  robustly-reliable  w.r.t.\ $\ellst$, $\rrst(S)\subseteq  A_{\text{ST}}$, where $A_{\text{ST}} = \{z\mid\cU^{-1}(z)\subseteq\operatorname{Agree}(\cH_0(S)) \wedge h(x)=h(z), \;\forall\;x\in \cU^{-1}(z), h\in \cH_0(S)\}.$
    \item \blue{There exists a robustly-reliable learner $\cL$ w.r.t.\ $\cU$ and $\ellia$, such that $\rria(S)\supseteq A_{\text{IA}} $, and for any $\cL$  robustly-reliable  w.r.t.\ $\ellia$, $\rria(S)\subseteq  A_{\text{IA}}$, where $A_{\text{IA}} =  (A_{\text{ST}} \cap \{z \mid h^*(z) = 1\})\cup \{ z \mid z \in \operatorname{Agree}(\cH_0(S)) \wedge h^*(z) = 0\}$}.
\end{enumerate}

\label{thm: upper bound rr for l3 }
\end{theorem}

\begin{proof}

We first establish part (a).
Given sample $S$, consider the learner $\cL$ which outputs $h^{\cL}_S=\argmin_{h\in\cH}\operatorname{err}_S(h)$ i.e. an ERM over $S$, and $a^{\cL}_S(z)=\bbI[z\in\operatorname{Agree}(\cH_0(S))]$. By realizability, $\operatorname{err}_S(h^{\cL}_S)\le \operatorname{err}_S(h^*)=0$, or $h^{\cL}_S\in\cH_0(S)$. We first show that $\cL$ is robustly-reliable. For $z\in\cX$, if $a^{\cL}_S(z)=1$, then $z\in \operatorname{Agree}(\cH_0(S))$. 
We have $h^*(z)=h^{\cL}_S(z)$ since the classifiers $h^*,h^{\cL}_S\in \cH_0(S)$ and $z$ lies in the agreement region of classifiers in $\cH_0(S)$. Thus, we have $\ellca^{h^*}(h^{\cL}_S,x,z)=0$ for any $x$ such that $z\in\cU(x)$. $\rrca(S)\supseteq \operatorname{Agree}(\cH_0(S))$ follows from the choice of $a^{\cL}_S(z)=\bbI[z\in\operatorname{Agree}(\cH_0(S))]$.

On the other hand, Let $z\in \operatorname{DIS}(H_0(S))$. There exist $h_1,h_2\in H_0(S)$ such that $h_1(z)\ne h_2(z)$. If possible,  let there be a robustly-reliable learner $\cL$ such that $z\in \rrca(S)$. That is, $a^{\cL}_S(z)=1$. We have $z\in\cU(z)$, and therefore $\ellca^{h^*}(h^{\cL}_S,z,z)=0$. But then we must have $h^{\cL}_S(z)=h^*(z)$ by definition of $\ellca$. We can set $h^*=h_1$ or $h^*=h_2$ since both are consistent with $S$. However, $h_1(z)\ne h_2(z)$, and therefore $h^{\cL}_S(z)\ne h^*(z)$ for one of the above choices for $h^*$, contradicting that $\cL$ is robustly-reliable.

This completes the proof of (a). Essentially the same argument may be used to establish (b), by substituting $\ellca$ with $\elltl$. We will now turn our attention to part (c).

Given sample $S$, consider the learner $\cL$ which outputs $h^{\cL}_S=\argmin_{h\in\cH}\operatorname{err}_S(h)$, that is an ERM over $S$, and $a^{\cL}_S(z)=\bbI[\cU^{-1}(z)\in\operatorname{Agree}(\cH_0(S))\wedge h^{\cL}_S(x_1)=h^{\cL}_S(x_2)\;\forall\;x_1,x_2\in \cU^{-1}(z)]$. By realizability, $\operatorname{err}_S(h^{\cL}_S)\le \operatorname{err}_S(h^*)=0$, or $h^{\cL}_S\in\cH_0(S)$. We first show that $\cL$ is robustly-reliable w.r.t. $\ellst$. For $z\in\cX$, if $a^{\cL}_S(z)=1$, then $\cU^{-1}(z)\subseteq\operatorname{Agree}(\cH_0(S))$, in particular $z\in \operatorname{Agree}(\cH_0(S))$. 
Moreover, by definition, for any $x$ that $z\in U(x)$, we have $h^{\cL}_S(z)=h^{\cL}_S(x)$ by construction of $a^{\cL}_S(z)$. Putting together, we have $h^{\cL}_S(z)=h^*(x)$ for any $x$ such that $z\in\cU(x)$. Thus, we have $\ellst^{h^*}(h^{\cL}_S,x,z)=0$ for any $x$ such that $z\in\cU(x)$.

On the other hand, for any robustly-reliable learner $\cL$, we will show that
\begin{equation*}
    \rrst(S) \subseteq \operatorname{Agree}(\cH_0(S)).
\end{equation*}

Let $z\in \operatorname{DIS}(H_0(S))$. There exists $h_1,h_2\in H_0(S)$ such that $h_1(z)\ne h_2(z)$. If possible,  let there be a robustly-reliable learner $\cL$ such that $z\in \rrst(S)$. That is, $a^{\cL}_S(z)=1$. We have $z\in\cU^{-1}(z)$, and therefore $\ellst^{h^*}(h^{\cL}_S,z,z)= \bbI[h_S^{\cL}(z) \neq h^*(z)] = 0$ which implies that $h^{\cL}_S(z)=h^*(z)$. But we can set $h^*=h_1$ or $h^*=h_2$ since both are consistent with $S$. However, $h_1(z)\ne h_2(z)$, and therefore $h^{\cL}_S(z)\ne h^*(z)$ for one of the above choices for $h^*$, contradicting that $\cL$ is robustly-reliable. Next, we will show that 
\begin{equation*}
    \rrst(S)\subseteq  \{z\mid\cU^{-1}(z)\subseteq\operatorname{Agree}(\cH_0(S)) \}.
\end{equation*}
We will prove this by contradiction, $z\in \operatorname{Agree}(H_0(S))$ but there exists $x' \in \cU^{-1}(z)$ such that $x' \not\in \operatorname{Agree}(H_0(S))$. Let there be a robustly-reliable learner $\cL$ such that $z\in \rrst(S)$. By definition, we have
$\ellst^{h^*}(h^{\cL}_S,x,z)=0$ for any $x$ that $z \in \cU(x)$. This implies that $\ellst^{h^*}(h^{\cL}_S,x',z)=0$ that is $h^{\cL}_S(z) = h^*(x')$. Because $x' \not\in \operatorname{Agree}(H_0(S))$, there exists $h_1,h_2\in H_0(S)$ such that $h_1(x')\ne h_2(x')$. We can set $h^* = h_1$ or $h^* = h_2$ since both are consistent with $S$. But $h_1(x')\ne h_2(x')$, and therefore $h^{\cL}_S(z)\ne h^*(z)$ for one of the above choices for $h^*$, contradicting that $\cL$ is robustly-reliable. Next, we will show that
\begin{equation*}
\rrst(S)\subseteq  \{z\mid\cU^{-1}(z)\subseteq\operatorname{Agree}(\cH_0(S)) \wedge h(x)=h(z), \;\forall\;x\in \cU^{-1}(z), h\in \cH_0(S)\}.
\end{equation*}
Let $z$ be a data point that $\cU^{-1}(z) \subseteq \operatorname{Agree}(H_0(S))$ but there exists $x' \in \cU^{-1}(z)$ that $h(x') \neq h(z)$ for $h \in H_0(S)$. Let there be a robustly-reliable learner that $z \in \rrst(S)$. This implies that $\ellst(h_S^{\cL},x,z) = 0$ for any $x$ that $z \in \cU(x)$. However, $\ellst(h_S^{\cL},x',z) = \bbI[h_S^{\cL}(z) \neq h^*(x')] = \bbI[h_S^{\cL}(z) \neq h_S^{\cL}(x')] \neq 0$, contradicting that $\cL$ is robustly-reliable.
\blue{

Finally,  the proof of part d) is similar to the proof of part c).
Given sample $S$, consider the learner $\cL$ which outputs $h^{\cL}_S=\argmin_{h\in\cH}\operatorname{err}_S(h)$, that is an ERM over $S$, and 
\begin{enumerate}
    \item if $h^{\cL}_S(z) = 1$, let $a^{\cL}_S(z)=\bbI[\cU^{-1}(z)\in\operatorname{Agree}(\cH_0(S))\wedge h^{\cL}_S(x_1)=h^{\cL}_S(x_2)\;\forall\;x_1,x_2\in \cU^{-1}(z)]$;
    \item if $h^{\cL}_S(z) = 0$, let $a^{\cL}_S(z)=\bbI[z \in\operatorname{Agree}(\cH_0(S))]$.
\end{enumerate}
 By realizability, $\operatorname{err}_S(h^{\cL}_S)\le \operatorname{err}_S(h^*)=0$, or $h^{\cL}_S\in\cH_0(S)$. We first show that $\cL$ is robustly-reliable w.r.t. $\ellia$. For $z\in\cX$, if $h^{\cL}_S(z) = 1$ and $a^{\cL}_S(z)=1$, then $\cU^{-1}(z)\subseteq\operatorname{Agree}(\cH_0(S))$, in particular $z\in \operatorname{Agree}(\cH_0(S))$. Moreover, by definition, for any $x$ that $z\in U(x)$, we have $h^{\cL}_S(z)=h^{\cL}_S(x)$ by construction of $a^{\cL}_S(z)$. Putting together, we have $h^{\cL}_S(z)=h^*(x)$ for any $x$ such that $z\in\cU(x)$. Thus, we have $\ellst^{h^*}(h^{\cL}_S,x,z)=0$ for any $x$ such that $z\in\cU(x)$. This also implies that $\ellia^{h^*}(h^{\cL}_S,x,z)=0$ for any $x$ such that $z\in\cU(x)$ since $\ellst$ implies $\ellia$.

For $z \in \cX$, if $h^{\cL}_S(z) = 0$ and $a^{\cL}_S(z) = 1$, we have $z \in \operatorname{Agree}(\cH_0(S))$ and $h^{\cL}_S(z) = h^*(z) = 0$. By the incentive-aware property of the adversary, any $x$ such that $z\in\cU(x)$, we can't have $h^*(x) = 1$ since the adversary has no incentive to make any perturbation in this case. Therefore, we have $h^*(x) = 0$ and $\cU_{\text{IA}}(h^*,x) = \cU(x)$. We have $\ellia^{h^*}(h^{\cL}_S,x,z)= \bbI[h(z) \neq h^*(x) \wedge z \in \cU_{\text{IA}}(x,h^*)] =0$. We can conclude that our learner $\cL$ is robustly-reliable w.r.t. $\ellia$.

Conversely, we will show that for any robustly-reliable learner $\cL$ w.r.t. $\ellia$, for any $\eta>0$,
\begin{equation*}
    \rria(S)\cap\{z \mid h^{\cL}_S(z) = 0 \} \subseteq \operatorname{Agree}(\cH_0(S)) \cap\{z \mid h^*(z) = 0 \}.
\end{equation*}
Since $z \in \cU_\text{IA}(z)$, for $z$ to lie in the robustly-reliable region, we need $\ellia^{h^*}(h^{\cL}_S,z,z) = \bbI[h^{\cL}_S(z) \neq h^*(z)] = 0$ that is $z$ must be reliable. By similar arguments from above,  we have the result. Next, we will show that,
\begin{equation*}
    \rria(S) \cap \{z \mid h^{\cL}_S(z) = 1 \}\subseteq  \{z\mid \cU^{-1}(z) \subseteq\operatorname{Agree}(\cH_0(S)) \}.
\end{equation*}
We will prove this by contradiction. Suppose $z\in \operatorname{Agree}(H_0(S))$, but there exists $x' \in \cU^{-1}(z)$ such that $x' \not\in \operatorname{Agree}(H_0(S))$. Let there be a robustly-reliable learner $\cL$ such that $z\in \rria(S)$ and $h^{\cL}_S(z) = 1$. Because $x' \not\in \operatorname{Agree}(H_0(S))$, there exists $h_1\in H_0(S)$ such that $h_1(x')=0$. We may have $h^* = h_1$ since $h_1$ is consistent with $S$. However, we have $\cU_{\text{IA}}(x', h_1) = \cU(x')$ and 
$$\ellia^{h^*}(h^{\cL}_S,x',z)= \bbI[h(z) \neq h^*(x') \wedge z \in \cU_{\text{IA}}(x')] = 1$$
which contradicts with $z$ lies in the robustly-reliable region. Furthermore, with a similar argument that we can't have $h^*(x') = 0$, we can show that the agreed label of any $x$ must be $1$,
\begin{align*}
    &\rria(S,\eta)\cap \{z \mid h^{\cL}_S(z) = 1 \}\\
&\subseteq  \{z\mid\cU^{-1}(z) \subseteq\operatorname{Agree}(\cH_0(S)) \wedge h(x)= 1, \;\forall\;x\in \cU^{-1}(z), h\in \cH_0(S)\}.\\
&= A_\text{ST} \cap \{z \mid h^*(z) = 1 \}
\end{align*}
This concludes that for any robustly-reliable learner $\cL$ with respect to $\ellia$, we have
$$
\rria(S) \subseteq (A_{\text{ST}} \cap \{z \mid h^*(z) = 1\})\cup \{ z \mid z \in \operatorname{Agree}(\cH_0(S)) \wedge h^*(z) = 0\}.
$$

}
\end{proof}

\noindent We can also define a safely-reliable region for general perturbations as follows.
\begin{definition}
\label{def: general safely reliable}
    (General safely-reliable region) Let $\cL$ be a robustly-reliable learner w.r.t.\ a perturbation function $\cU$ for sample $S$, concept space $\cH$ and robust loss function $\ell$. The safely-reliable region of a learner $\cL$ is defined as  $\text{SR}^{\cL}(S) = \{x \in \cX \mid  \cU(x) \subseteq \text{RR}^{\cL}(S)\}$.
\end{definition}

\section{Additional proof details for safely-reliable region}
\label{appendix: safely-reliable}

\begin{lemma}
\label{lemma: ball in agreement}
Let $\cD$ be isotropic log-concave over $\R^d$ and $\cH = \{h: x \to \operatorname{sign}(\langle w_h, x\rangle ) \mid w_h \in \R^d, \lVert w_h \rVert_2 = 1\}$ be the class of linear separators. Let $\B(\cdot, \eta)$ be a $L_2$ ball perturbation with radius $\eta$. For $S \sim \cD^m$, for  $m = \cO(\frac{1}{\varepsilon^2}(\operatorname{VCdim}(\cH) + \ln\frac{1}{\delta}))$, with probability at least $1-\delta$,  we have 
\begin{equation*}
    \Pr(\{x \mid \B(x,\eta) \subseteq \operatorname{Agree}(\cH_0(S) \}) \geq 1 -  2\eta - \Tilde{\cO}(\sqrt{d}\varepsilon).
\end{equation*}
\end{lemma}
\begin{proof}
(Proof of Lemma \ref{lemma: ball in agreement})
    From uniform convergence (Theorem 4.1 \cite{anthony1999neural}), for $S \sim \cD^m$, for  $m = \cO(\frac{1}{\varepsilon^2}(\operatorname{VCdim}(\cH) + \ln\frac{1}{\delta}))$, with probability at least $1-\delta$, we have
    $\operatorname{Agree}(B_\cD^{\cH}(h^*, \varepsilon)) \subseteq \operatorname{Agree}(\cH_0(S))$. From \cite{balcan2013active} (Theorem 14), for linear separators on a log-concave distribution $A = \{x: \lVert x\rVert_2 < \alpha \sqrt{d}\} \cap \{ x: |\langle w_{h^*},x \rangle| \ \geq C_1\alpha\varepsilon\sqrt{d}\} \subseteq \operatorname{Agree}(B_\cD^{\cH}(h^*, \varepsilon))$
for some constant $C_1$. We claim that for any $x \in A_\eta: = \{x: \lVert x\rVert_2 < \alpha \sqrt{d} - \eta\} \cap \{ x: |\langle w_{h^*},x \rangle| \ \geq C_1\alpha\varepsilon\sqrt{d}+ \eta\}$, we have $\B(x,\eta) \subseteq A$. Let $x \in A_\eta$, consider $z\in \B(x,\eta)$. We have $||z||_2 \leq ||z-x||_2 + \lVert x\rVert_2 \leq \eta + \alpha\sqrt{d} - \eta = \alpha\sqrt{d}$ and $|\langle w_{h^*},z \rangle| \geq |\langle w_{h^*},x \rangle| - |\langle w_{h^*},z-x \rangle|\geq C_1\alpha\varepsilon\sqrt{d} + \eta - ||z-x|| \geq C_1\alpha\varepsilon\sqrt{d}.$ Therefore, $z \in A$ for any $z\in \B(x,\eta)$ which implies that for any $x\in A_\eta$, $\B(x,\eta) \subseteq A$ which also implies that $A_\eta \subseteq \{x\mid \B(x,\eta) \subseteq \operatorname{Agree}(\cH_0(S))\}$.  We can bound the probability mass of $A_\eta$ with the following fact on isotropic log-concave distribution $\mathcal{D}$ over $\mathbb{R}^d$ \cite{lovasz2007geometry}: 1)  $\operatorname{Pr}_{x \sim \mathcal{D}}(\|x\| \geq \alpha \sqrt{d}) \leq e^{-\alpha+1}$, 2) When $d=1$ $\operatorname{Pr}_{x \sim \mathcal{D}}(x \in[a, b]) \leq|b-a|$ and 3) The projection $\langle w_{h^*},x \rangle$ follows a 1-dimensional isotropic log-concave distribution. We have
$
    \Pr_{x\sim\cD}(A_\eta)
    \geq 1 - \Pr_{x\sim\cD}(\{x: ||x|| \geq \alpha \sqrt{d} - \eta\}) - \Pr_{x\sim\cD}(\{ x: |\langle w_{h^*},x \rangle| \ \leq C_1\alpha\varepsilon\sqrt{d} + \eta \})
    \geq 1 - e^{-\left(\alpha - \frac{\eta}{\sqrt{d}}\right) + 1} - 2C_1\alpha\varepsilon\sqrt{d} - 2\eta
    = 1 - 2\eta - \Tilde{\cO}(\sqrt{d}\varepsilon).
$
The final line holds when we set $\alpha = \ln(\frac{1}{\sqrt{d}\varepsilon})$.
\end{proof}

\begin{figure}
    \centering
    \includegraphics[width = 0.4\textwidth]{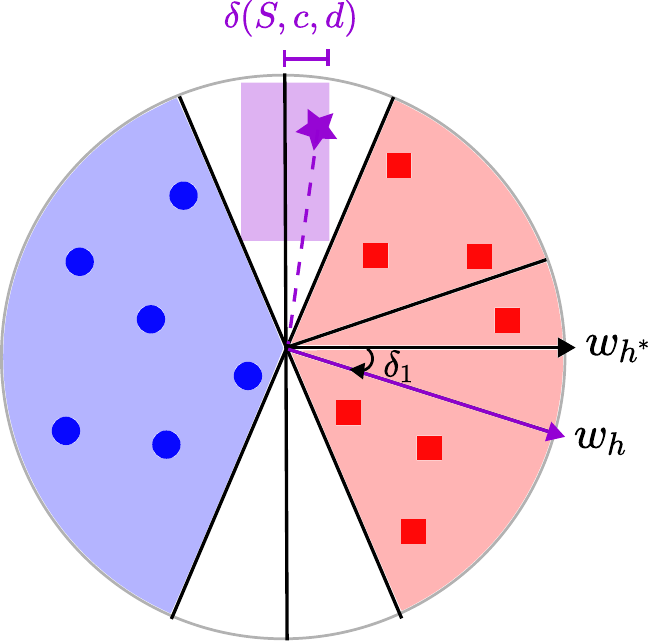}
    \caption{For any set of samples $S$ (blue and red points) with no point lying on the decision boundary of  a linear separator $h^*$, 
    for any $0 < c < d$, there exists an area  around the decision boundary of $h^*$ (formally defined as $\{x \in \R^d \mid c \leq \lVert x \rVert \leq d, |\langle w_{h^*}, x \rangle| \leq \delta(S,c,d) \}$, illustrated by a purple rectangle) such that for any point (purple star) in this area, there exists a hypothesis $h$ that agree with $h^*$ on $S$ but disagree with $h^*$ at that point}
    \label{fig: Lemma 12 fig}
\end{figure}
\begin{proof}
(Proof of Lemma \ref{lemma: agreement contain large margin}) First, we will show that for any sample $S \sim \cD^m$ with no points lying on the decision boundary of $h^*$, there exists  a constant $\delta_1(S)$ such that for any $h$ with a small enough angle to $h^*$, $\theta(w_h, w_{h^*}) \leq \delta_1$, $h$ must have the same prediction as $h^*$ on $S$ that is  $h \in \cH_0(S).$ Since there is no point lying on the decision boundary, we have $\min_{x \in S}\frac{|\langle w_{h^*}, x\rangle|}{\lVert x \rVert} > 0$. 

For $\delta_1$, such that  $0 < \delta_1 < \min_{x \in S}\frac{|\langle w_{h^*}, x\rangle|}{\lVert x \rVert} $ ,  if $\theta(w_h, w_{h^*}) \leq \delta_1$,  for any $x \in S$,
\begin{align*}
    |\langle w_{h^*}, x\rangle - \langle w_{h}, x\rangle| &\leq \lVert w_{h^*} - w_{h}\rVert \lVert x \rVert \\
    &\leq \theta(w_{h^*}, w_h) \lVert x \rVert\\
    &< |\langle w_{h^*}, x\rangle|.
\end{align*}
The second to last inequality holds due to the fact that the arc length cannot be smaller than corresponding chord length, and the last inequality follows from the assumption  $\theta(w_h, w_{h^*}) \leq \delta_1$. This implies that $\langle w_{h^*}, x\rangle  \langle w_{h}, x\rangle > 0$ and $h \in \cH_0(S)$.  Now, consider any $x$ such that $c \leq \lVert x \rVert \leq d$. We will show that there exists a constant $\delta=\delta(S,c,d)$ such that if the margin of $|\langle w_{h^*}, x \rangle|$ is smaller than $\delta$ then $x \not\in \operatorname{Agree}(\cH_0(S))$.  Since $|\langle w_{h^*}, x \rangle| = \lVert x \rVert \cos(\theta(w_{h^*},x)) \geq c \cos(\theta(w_{h^*},x))$, $\delta > |\langle w_{h^*}, x \rangle|$ implies that $\cos(\theta(w_{h^*},x)) < \frac{\delta}{c}$ that is the angle between $w_{h^*}$ and $x$ is almost $\frac{\pi}{2}$. Intuitively, we claim that if $\delta$ is small enough then there exists $h \in \cH_0(S)$ such that $h(x) \neq h^*(x)$.  Without loss of generality, let $\langle w_{h^*}, x \rangle > 0$  ($\theta(w_{h^*},x) < \frac{\pi}{2}$). We will show that if $\theta(w_{h^*},x)$ is close enough to $\frac{\pi}{2}$, we can rotate $w_{h^*}$ to $w_h$ with a small enough angle so that  $\theta(w_h, w_{h^*}) \leq \delta_1 $ but  $\langle w_{h}, x \rangle < 0$ ($\theta(w_h,x) > \frac{\pi}{2}$) as illustrated in Figure \ref{fig: Lemma 12 fig}. Formally, we consider $w_h = \frac{w_{h^*} - \lambda x}{\lVert w_{h^*} - \lambda x \rVert}$ for some $\lambda > 0$ (to be specified). We will show that there exists $\lambda$ such that 1) $\langle w_{h}, x \rangle < 0$, and 2) $\theta(w_h, w_{h^*}) \leq \delta_1 $.  The first condition corresponds to $\lambda > \frac{\langle x ,w_{h^*} \rangle}{\lVert x \rVert^2}$.  The second condition leads to the following inequality
$$
\frac{\langle w_{h^*} - \lambda x, w_{h^*}\rangle}{\lVert w_{h^*} - \lambda x \rVert} \geq \cos(\delta_1)$$
$$
\frac{1 - \lambda \langle x, w_{h^* }\rangle}{\sqrt{1 - 2 \lambda \langle x, w_{h^*} \rangle + \lVert x \rVert^2 \lambda^2 }} \geq \cos(\delta_1)
$$
Assume that $\lambda < \frac{1}{\langle x, w_{h^* }\rangle}$, the inequality is equivalent to
$$
(1 - \lambda \langle x, w_{h^* }\rangle)^2
 \geq \cos^2(\delta_1)(1 - 2 \lambda \langle x, w_{h^*} \rangle + \lVert x \rVert^2 \lambda^2 )
$$
$$(\cos^2(\delta_1)\lVert x \rVert^2 - \langle x, w_{h^* }\rangle^2)\lambda^2 + 2\lambda \langle x, w_{h^* }\rangle \sin^2(\delta_1) - \sin^2(\delta_1) \leq 0.$$
Solving this inequality leads to 
$$\lambda \leq \lambda_{\text{max}} = \frac{-2\sin^2(\delta_1)\langle x, w_{h^* }\rangle + 2\sin(\delta_1)\cos(\delta_1)\sqrt{(\lVert x \rVert^2 - \langle x, w_{h^* }\rangle^2)}}{2(\cos^2(\delta_1)\lVert x \rVert^2 - \langle x, w_{h^* }\rangle^2)}.$$
Therefore, there exists $\lambda$ that satisfies both of conditions 1), 2) if $\lambda_{\text{max}} > \frac{\langle x, w_{h^* }\rangle}{\lVert x \rVert^2}$. Finally, we claim that if $|\langle x, w_{h^* }\rangle| \leq \delta(S,c,d) \leq \frac{c^2 \tan(\delta_1)}{\sqrt{(d + d\tan(\delta_1))^2 + (c^2 \tan(\delta_1))^2}}$  then $\lambda_{\text{max}} > \frac{\langle x, w_{h^* }\rangle}{\lVert x \rVert^2}$.  For $x$ with $|\langle x, w_{h^* }\rangle| \leq \delta$ , we have  $\frac{\langle x, w_{h^* }\rangle}{\lVert x \rVert^2} \leq \frac{\delta}{\lVert x \rVert^2} \leq \frac{\delta}{c^2}$ and also
\begin{align*}
    \lambda_{\text{max}} &= \frac{-2\sin^2(\delta_1)\langle x, w_{h^* }\rangle + 2\sin(\delta_1)\cos(\delta_1)\sqrt{(\lVert x \rVert^2 - \langle x, w_{h^* }\rangle^2)}}{2(\cos^2(\delta_1)\lVert x \rVert^2 - \langle x, w_{h^* }\rangle^2)}\\
    &> \frac{-\sin^2(\delta_1)\langle x, w_{h^* }\rangle + \sin(\delta_1)\cos(\delta_1)\sqrt{(\lVert x \rVert^2 - \langle x, w_{h^* }\rangle^2)}}{\cos^2(\delta_1)\lVert x \rVert^2 }\\
    &= \frac{-\sin^2(\delta_1)\frac{\langle x, w_{h^* }\rangle}{\lVert x \rVert^2} + \sin(\delta_1)\cos(\delta_1)\sqrt{(1 - (\frac{\langle x, w_{h^* }\rangle}{\lVert x \rVert})^2)}}{\cos^2(\delta_1) }\\
    &\geq \frac{-\sin^2(\delta_1)\frac{\delta}{c^2} + \sin(\delta_1)\cos(\delta_1)\sqrt{(1 - (\frac{\delta}{c^2})^2)}}{\cos^2(\delta_1)d }.\\
\end{align*}
The last inequality follows from $\frac{\langle x, w_{h^* }\rangle}{\lVert x \rVert^2} \leq  \frac{\delta}{c^2}$. It is sufficient to show that
\begin{align*}
    &\frac{-\sin^2(\delta_1)\frac{\delta}{c^2} + \sin(\delta_1)\cos(\delta_1)\sqrt{(1 - (\frac{\delta}{c^2})^2)}}{\cos^2(\delta_1)d }\geq \frac{\delta}{c^2}\\
    &\Leftrightarrow \quad -\tan^2(\delta_1)\frac{\delta}{c^2} + \tan(\delta_1)\sqrt{(1 - (\frac{\delta}{c^2})^2)} \geq d\frac{\delta}{c^2}\\
    &\Leftrightarrow \quad  \tan(\delta_1)\sqrt{(1 - (\frac{\delta}{c^2})^2)} \geq (d + \tan^2(\delta_1))\frac{\delta}{c^2} \\
    &\Leftrightarrow \quad \delta(S,c,d) \leq \frac{c^2\tan(\delta_1)}{\sqrt{(d + \tan^2(\delta_1))^2 + c^2\tan^2(\delta_1)}}.
\end{align*}

\end{proof}

\begin{proof}
\begin{figure}[t]
    \centering
    \includegraphics[width = 0.3\textwidth]{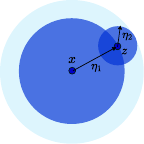}
    \caption{The safely-reliable region contains any point that retains a reliability radius of at least $\eta_2$ even after being attacked by an adversary with strength $\eta_1$.}
    \label{fig:Safely reliable region}
\end{figure}
    (Proof of Theorem \ref{thm: SR-l1l2l3})
     We know that for $\ellca, \elltl$, the robustly-reliable region is the same as the reliable region. This is also a reason why the probability mass does not depend on $\eta_2$. Consider the optimal robustly-reliable learner $\cL$, we have
    $\rrca(S,\eta_2) = \rrtl(S,\eta_2) = \operatorname{Agree}(\cH_0(S))$ (Theorem \ref{thm:rr-eta-l1}). On the other hand, $\rrst(S,\eta_2) =  \{z\mid\B(z,\eta_2)\subseteq\operatorname{Agree}(\cH_0(S)) \wedge h(x)=h(z), \;\forall\;x\in \B(z,\eta_2), h\in \cH_0(S)\}$ (Theorem \ref{thm:mball-l3-ub}).
    \textbf{a)} Since $\srtl = \{ x \mid \B(x,\eta_1) \subseteq \operatorname{Agree}(\cH_0(S)) \}$. Applying Lemma \ref{lemma: ball in agreement}, we have the first result. 
\textbf{b)}
     Recall that $\srca= \{x \in \cX \mid  \B(x,\eta_1)\cap \{z\mid h^*(z) = h^*(x) \} \subseteq \operatorname{Agree}(\cH_0(S))\}$. We will show that for any $x \in \srca$, $\B(x,\eta_1) = \B(x,\eta_1)\cap \{z\mid h^*(z) = h^*(x) \}$ by contradiction. Let $x \in \srca$ and $\B(x,\eta_1)$ contain two points with a different label, this implies that this ball must contain the decision boundary of $h^*$ ($\B(x,\eta_1) \cap \{x \mid \langle w_h^*, x \rangle  = 0\} \neq \emptyset$). The ball must also contain a point that has the same label as $x$ with an arbitrarily small margin w.r.t. $h^*$. For any $a > 0$, there exists $z \in \B(x,\eta_1)\cap \{z\mid h^*(z) = h^*(x) \}$ with $|\langle z, w_{h^*} \rangle| < a$. This is impossible because by Lemma \ref{lemma: agreement contain large margin} the agreement region $\operatorname{Agree}(\cH_0(S))$ can not contain point with arbitrarily small margin if $S$ does not contain any point on the decision boundary of $h^*$. This event has a probability $1$ as the projection $\langle w_{h^*}, x \rangle$ also follows a log-concave distribution which implies that $\Pr(\langle w_{h^*}, x \rangle = 0) = 0$. Therefore, with probability $1$, $\B(x,\eta_1)$ must contain points with the same label and we can conclude that $\srca = \srtl$. \textbf{c)} Similarly, by Lemma \ref{lemma: agreement contain large margin}, we can show that if $\B(z,\eta_2)\subseteq\operatorname{Agree}(\cH_0(S))$, $\B(z,\eta_2)$ then every point in $\B(z,\eta_2)$ must have same label with probability $1$. Therefore, $\rrst(S,\eta_2) =  \{z\mid\edit{\B^o}(z,\eta_2)\subseteq\operatorname{Agree}(\cH_0(S)) \}$. We have $\srst = \{x \in \cX \mid  \B_\cM(x,\eta_1) \subseteq \{z\mid\edit{\B^o}(z,\eta_2)\subseteq\operatorname{Agree}(\cH_0(S)) \}\} = \{x \in \cX \mid  \edit{\B^o}_\cM(x,\eta_1 + \eta_2) \subseteq\operatorname{Agree}(\cH_0(S))\}$ by a triangle inequality (see Figure \ref{fig:Safely reliable region}). Applying Lemma \ref{lemma: ball in agreement}, we have the result. \blue{\textbf{d)} With the result above we have $\rria(S, \eta_2) = (\{z\mid\B(z,\eta_2)\subseteq\operatorname{Agree}(\cH_0(S)) \}\cap\{z \mid h^*(z) = 1\}) \cup (\operatorname{Agree}(\cH_0(S)) \cap \{z \mid h^*(z) = 0\})$. Recall that $\sria = \{x \in \cX \mid  h^*(x) = 0 \wedge \B_\cM(x,\eta_1) \subseteq \rria(S,\eta_2)\} \cup \{x \in \cX \mid  h^*(x) = 1 \wedge x \in  \rrca(S,\eta_2))  \}$. Therefore, we have
     $\sria = (\{z\mid\B(z,\eta_2)\subseteq\operatorname{Agree}(\cH_0(S)) \}\cap\{z \mid h^*(z) = 1\}) \cup 
     (\{z\mid\B(z,\eta_1)\subseteq\operatorname{Agree}(\cH_0(S)) \}\cap\{z \mid h^*(z) = 0\})$ We can conclude the result by applying Lemma \ref{lemma: ball in agreement} and symmetry.
     }
\end{proof}

\section{Safely-reliable region for classifiers with smooth boundaries}
\label{appendix: safely-reliable region for classifiers with smooth boundaries}

We also bound the probability mass of the safely-reliable region for more general concept spaces beyond linear separators.  Specifically, we consider classifiers with smooth boundaries in the sense of \cite{van1996weak}.

\begin{definition}[$\alpha$-norm]\label{def:alpha-norm} Let $f:C\rightarrow\R$ be a function on $C\subset\R^d$, and let $\alpha\in\R^+$. For $k=(k_1,\dots,k_d)\in\Z^d_{\ge0}$, let $||k||_1=\sum_{i=1}^dk_i$ and let $D^k=\frac{\partial^k}{\partial^{k_1}x_1\dots \partial^{k_d}x_d}$. We define $\alpha$-norm of $f$ as
    $$||f||_\alpha:=\edit{\max_{\substack{\|k\|_1 < \lceil\alpha\rceil\\ k \neq 0}}} \sup_{x\in {C}}|D^kf(x)|+\max_{||k||_1=\lceil\alpha\rceil-1}\sup_{x\ne x'\in {C}}\frac{|D^kf(x)-D^kf(x')|}{|x-x'|^{\alpha-\lceil\alpha\rceil+1}}.$$
   
\end{definition}

\noindent We define $\alpha$th order smooth functions to be those which have a bounded $\alpha$-norm. More precisely, we define the class of $\alpha$th order smooth functions $F_\alpha^C := \{f \mid ||f||_\alpha \le C\}$. For example, $1$st order smoothness corresponds to Lipschitz continuity. We now define concept classes with smooth classification boundaries.

\begin{definition}[Concepts with Smooth Classification Boundaries, \cite{wang2011smoothness}] A set of concepts $\cH_\alpha^C$ defined on $\cX = [0,1]^{d+1}$ is said to have $\alpha$th order smooth classification boundaries, if for every $h\in\cH_\alpha^C$ the classification boundary is  the graph of function $x_{d+1} = f(x_1,\dots,x_d)$, where $f \in F_\alpha^C$ and $(x_1,\dots,x_{d+1})\in\cX$ i.e. the predicted label is given by $\sign(x_{d+1} - f(x_1,\dots, x_d))$.
\label{def:smooth-classifiers}
\end{definition}

\noindent If we further assume that the probability density may be upper and lower bounded by some absolute positive constants (i.e.\ ``nearly'' uniform density), we can bound the safely-reliable region of our learner even in this setting. We start with analogues of Lemmas \ref{lemma: ball in agreement} and \ref{lemma: agreement contain large margin} for concepts with smooth classification boundaries.

We first bound the probability mass of points $x$ for which $\B(x,\eta)$ is contained in the agreement region of sample-consistent classifiers. We use the Lipschitzness of smooth functions to show that such point $x$ must lie outside of a `ribbon' around the boundary of target concept $h^*$, and adapt and extend the arguments of \cite{wang2011smoothness} to bound the probability mass of this ribbon.

\begin{lemma}
\label{lemma: ball in agreement - smooth boundaries}
  Let the instance space be $\cX = [0,1]^{d+1}$ and $\cD$ be a distribution over $\cX$ with a ``nearly'' uniform density where there exist positive constants $0<a<b$ such that $a\le p(x)\le b$ for all $x \in[0,1]^{d+1}$ when $p(x)$ is the probability density of $\cD$. Let $\cH_\alpha^C$ be the hypothesis space of concepts with smooth classification boundaries with $d < \alpha<\infty$, and $\B(\cdot, \eta)$ be a $L_2$ ball perturbation with radius $\eta$. For $S \sim \cD^m$, for  $m = \cO(\frac{1}{\varepsilon^2}(\operatorname{VCdim}(\cH) + \ln\frac{1}{\delta}))$, with probability at least $1-\delta$, we have 
\begin{equation*}
    \Pr(\{x \mid \B(x,\eta) \subseteq \operatorname{Agree}(\cH_0(S)\}) \geq 1 - 2b (C+1)\eta - {\cO}\left( {b}{a^{-\frac{\alpha}{d+\alpha}}}\varepsilon^{\frac{\alpha}{d+\alpha}}\right).
\end{equation*}
\end{lemma}
\begin{proof}
By uniform convergence (Theorem 4.1, \cite{anthony1999neural}), for $S \sim \cD^m$, for  $m = \cO(\frac{1}{\varepsilon^2}(\operatorname{VCdim}(\cH) + \ln\frac{1}{\delta}))$, with probability at least $1-\delta$, we have
    $\operatorname{Agree}(B_\cD^{\cH}(h^*, \varepsilon)) \subseteq \operatorname{Agree}(\cH_0(S))$. Therefore, it suffices to lower bound $\pi:=\Pr\{x \mid \B(x,\eta) \subseteq \operatorname{Agree}(B_\cD^{\cH}(h^*, \varepsilon))\}$. 
Let $h^*\in\cH_\alpha^C$ be the target concept and denote $\x=(x_1,\dots,x_d)\in[0,1]^d$.  Recall that the predicted label of $(\x, x_{d+1})$ from $h,h^*$ is given by $\sign(x_{d+1} - f_h(\x))$ and $\sign(x_{d+1} - f_{h^*}(\x))$ respectively. Therefore, $h,h^*$ would disagree on $(\x, x_{d+1})$ when $x_{d+1}$ lies between $f_h(\x)$ and $f_{h^*}(\x)$. Denote $\Phi_h(\x)= |\int_{f_{h^*}(\x)}^{f_h(\x)}p(\x,x_{d+1})dx_{d+1}|$ be the probability mass of points that $h$ disagree with $h^*$ over $(\x, x_{d+1})$ for a fixed $\x \in [0,1]^d$. With this notation, the probability mass of points $(\x, x_{d+1})$ that $h$ and $h^*$ disagree with is given by $\int_{[0,1]^d}|{\Phi}_h(\x)|d\x$.  Furthermore, from   $ a\le p(\x,x_{d+1})\le b$, we know that $a|f_h(\x)-f_{h^*}(\x)|\le |\Phi_h(\x)|\le b|f_h(\x)-f_{h^*}(\x)|.$

Consider  $h\in B_\cD(h^*,\varepsilon)$, we have $\int_{[0,1]^d}|{\Phi}_h(\x)|d\x\le \varepsilon$.  This implies that 
$\int_{[0,1]^d}|f_h(\x)-f_{h^*}(\x)|d\x\le  \int_{[0,1]^d}|\frac{\Phi(\x)}{a}|d\x\le \frac{\varepsilon}{a}.$
    Since the classification boundaries are assumed to be $\alpha$th order smooth with $\alpha>d$, Lemma 11 of \cite{wang2011smoothness}  implies that $||f_h-f_{h^*}||_{\infty}=O\left((\frac{\varepsilon}{a})^{\frac{\alpha}{d+\alpha}}\right)$  where $||g||_\infty:=\sup_{\x\in[0,1]^d}|g(\x)|$. Consider
$$
 1 - \pi = \Pr_{x\sim \cD_\cX}( \exists z \in \B(x, \eta), \exists h \in B(h^*, \varepsilon), h(z) \neq h^*(z) ).
$$
Recall that for $z = (\z, z_{d+1})$, $h(z) \neq h^*(z)$ when $z_{d+1}$ lies between $h(\z), h^*(\z)$ which implies $f_{h^*}(\z) - |f_{h^*}(\z) - f_{h}(\z)| < z_{d+1} < f_{h^*}(\z) +  |f_{h^*}(\z) - f_{h}(\z)|$. Since $z \in \B(x, \eta)$ and the boundary functions $f_h$ are $C$-Lipschitz (by Definition \ref{def:alpha-norm}) we have $|f_h(\z) - f_{h}(\x)| \leq C\lVert \z - \x \rVert \leq C\eta$ and $|z_{d+1} - x_{d+1}| \leq \eta$. This implies that $f_{h^*}(\x) - |f_{h^*}(\z) - f_{h}(\z)| - (C+1)\eta < x_{d+1} < f_{h^*}(\x) +  |f_{h^*}(\z) - f_{h}(\z)| + (C+1)\eta$. We are interested in the set $\{ x \mid \exists z \in \B(x, \eta), \exists h \in B(h^*, \varepsilon), h(z) \neq h^*(z) \}$, this leads to the inequality
$$f_{h^*}(\x) - D - (C+1)\eta < x_{d+1} < f_{h^*}(\x) +  D + (C+1)\eta$$ when
$D = \sup_{\substack{z \in \B(x, \eta)\\ h \in B(h^*, \varepsilon)}}|f_{h^*}(\z) - f_{h}(\z)| \leq \sup_{\substack{h \in B(h^*, \varepsilon)}}\lVert f_{h^*} - f_{h}\rVert_\infty = O\left((\frac{\varepsilon}{a})^{\frac{\alpha}{d+\alpha}}\right)$. Therefore,
\begin{align*}
     1 - \pi &\leq \int_{[0,1]^d}\int_{f_{h^*}(\x) - D - (C+1)\eta}^{f_{h^*}(\x) +  D + (C+1)\eta}p(\x, x_{d+1}) dx_{d+1}d\x\\
     &\leq 2b((C+1)\eta + D)\\
     &= 2b(C+1)\eta+O\left({b}{a^{-\frac{\alpha}{d+\alpha}}}\varepsilon^{\frac{\alpha}{d+\alpha}}\right).
\end{align*}

\end{proof}

\noindent The above bound immediately implies a bound on the probability mass of the safely-reliable region for $\elltl$, in combination with our previous results. The following lemma allows us to easily handle the extension of our results to losses $\ellca$ and $\ellst$ as well, where the safely-reliable region involves additional constraints.
   \begin{figure}
        \centering
        \includegraphics[width = 0.5\textwidth]{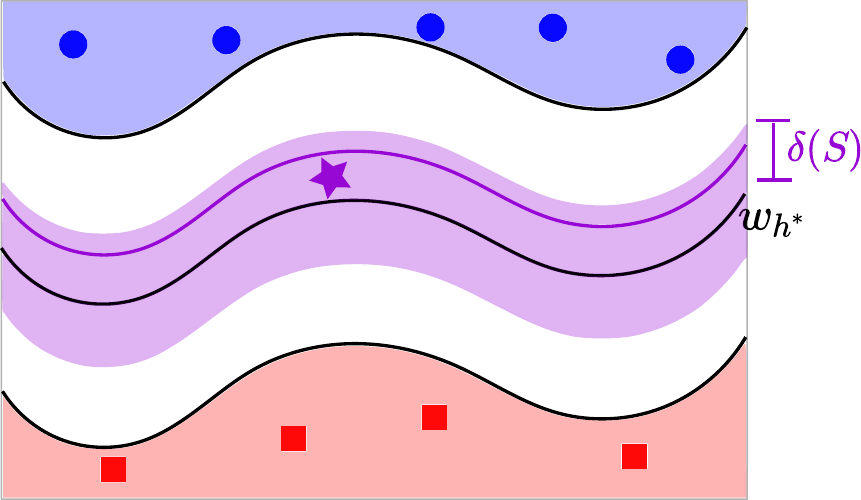}
        \caption{For any set of samples $S$ (blue and red points) with no point lying on the decision boundary of  a concept with smooth boundary $h^*$, 
     there exists a band  around the decision boundary of $h^*$ (formally defined as $\{x \in \R^d \mid |f_{h^*}(\x) - x_{d+1}| < \delta \}$ (purple)) such that for any point (purple star) in this area, there exists a hypothesis $h$ (by translation) that agree with $h^*$ on $S$ but disagree with $h^*$ at that point.}
        \label{fig:lemma 12 smooth}
    \end{figure}
\begin{lemma}
    Let the instance space be $\cX = [0,1]^{d+1}$ and $\cD$ be a distribution over $\cX$. Let $\cH_\alpha^C$ be the hypothesis space of concepts with smooth classification boundaries with $d < \alpha<\infty$. For $h^*\in \cH^C_\alpha$, for a set of samples $S \sim \cD^m$ such that there is no data point in $S$ that lies on the decision boundary, there exists $\delta(S) > 0$ such that for any $x = (\x, x_{d+1})$ with $|f_{h^*}(\x) - x_{d+1}| < \delta$, we have $x\not\in \operatorname{Agree}(\cH_0(S))$.
    \label{lemma: agreement contain large margin, smooth}
\end{lemma}
\begin{proof} Note that translation preserves smoothness, i.e.\ any concept $h_t$ with $f_{h_t}(\x) = f_{h^*}(\x) + t$ would also lie in $\cH^C_\alpha$. If $|t| < \delta(S) = \min_{x \in S} |f_{h^*}(\x) - x_{d+1}|$,  $f_{h_t}(\x) - x_{d+1}$ must have the same sign as   $(f_{h^*}(\x) - x_{d+1})$ for any $x\in S$, that is $h_t$ would agree with $h^*$ on every point in $S$. Furthermore, for any $x$ with $|f_{h^*}(\x) - x_{d+1}| < \delta$, we can always translate $h^*$ to $h_t$ that $h_t(x) \neq h^*(x)$ (see Figure \ref{fig:lemma 12 smooth}). Therefore, $x\not\in \operatorname{Agree}(\cH_0(S))$.
\end{proof}
\noindent Equipped with the above lemmas, we establish bounds on the probability mass of the safely reliable region for concept classes with smooth classification boundaries.

\begin{theorem}
\label{thm:mball-smooth}
    Let the instance space be $\cX = [0,1]^{d+1}$ and $\cD$ be a distribution over $\cX$ with a ``nearly'' uniform density where there exist positive constants $0<a<b$ such that $a\le p(x)\le b$ for all $x \in[0,1]^{d+1}$ when $p(x)$ is the probability density of $\cD$. Let $\cH_\alpha^C$ be the hypothesis space of concepts with smooth classification boundaries with $d < \alpha<\infty$, and $\B(\cdot, \eta)$ be a $L_2$ ball perturbation with radius $\eta$. For $S \sim \cD^m$, for  $m = \cO(\frac{1}{\varepsilon^2}(\operatorname{VCdim}(\cH) + \ln\frac{1}{\delta}))$ with no point lying on the decision boundary of $h^*$, for an optimal robustly-reliable learner $\cL$, we have
\begin{enumerate}[label=(\alph*),leftmargin=*,topsep=0pt,partopsep=1ex,parsep=1ex]\itemsep=-4pt
    \item $\Pr(\srtl) \geq  1 - 2b(C+1)\eta_1 - {\cO}\left({b}{a^{-\frac{\alpha}{d+\alpha}}}\varepsilon^{\frac{\alpha}{d+\alpha}}\right)$ with probability $1 - \delta$,
    \item $\srca = \srtl$,
    \item $\Pr(\srst) \geq  1 - 2b(C+1)(\eta_1 + \eta_2) - {\cO}\left({b}{a^{-\frac{\alpha}{d+\alpha}}}\varepsilon^{\frac{\alpha}{d+\alpha}}\right)$ with probability $1 - \delta$,
    \item \blue{$\Pr(\sria) \geq  1 - b(C+1)(\eta_1 + \eta_2) - {\cO}\left({b}{a^{-\frac{\alpha}{d+\alpha}}}\varepsilon^{\frac{\alpha}{d+\alpha}}\right)$ with probability $1 - \delta$.}
\end{enumerate}

\end{theorem}

\begin{proof}
From Lemma \ref{lemma: agreement contain large margin, smooth}, the agreement region does not contain points that are arbitrarily close to the boundary of $h^*$. Therefore, we can remove the additional conditions on labels for $\srca$ and $ \srst$, that is $\srca = \srtl = \{ x \mid \B(x,\eta_1) \subseteq \operatorname{Agree}(\cH_0(S)) \}$ and $\srst = \{x \in \cX \mid  \B_\cM(x,\eta_1) \subseteq \{z\mid\B(z,\eta_2)\subseteq\operatorname{Agree}(\cH_0(S)) \}\} = \{x \in \cX \mid  \B_\cM(x,\eta_1 + \eta_2) \subseteq\operatorname{Agree}(\cH_0(S))\}$. \blue{Similarly, we have $\sria = (\{z\mid\B(z,\eta_2)\subseteq\operatorname{Agree}(\cH_0(S)) \}\cap\{z \mid h^*(z) = 1\}) \cup 
     (\{z\mid\B(z,\eta_1)\subseteq\operatorname{Agree}(\cH_0(S)) \}\cap\{z \mid h^*(z) = 0\})$. }The result follows from Lemma \ref{lemma: ball in agreement - smooth boundaries}.

\end{proof}

\section{\edit{More on optimization-based implementation}}\label{appendix: computational efficiency}
\blue{
It is possible to extend the optimization objective to a wide range of hypothesis classes under the following assumption.
\begin{assumption}
For a hypothesis class $\cH$, we assume that for  any $h^* \in \cH$, a set of data points $S$ labeled by $h^*$ then for any points $x,y$ that $h^*(x) \neq h^*(y)$, the line that connects between $x,y$ must pass through a disagreement region of $\cH_0(S)$,
$$ \{\lambda x + (1-\lambda) y \mid \lambda \in [0,1] \} \cap \operatorname{DIS}(\cH_0(S)) \neq \emptyset.$$
\label{assumption: flexible class}
\end{assumption}
For example, a class of linear separators and a class of classifiers with smooth boundaries satisfies this assumption.
\begin{lemma}
    Let $\cH$ be a hypothesis class and $\cD$ be a distribution over $\R^d$. If $\cH$ satisfies Assumption \ref{assumption: flexible class} then  for a set of samples $S\sim \cD^m$, the \edit{square} reliability radius of a test point $z$ is given by
\begin{align*}
    \min_{h,h',z'} ||z-&z'||^2\\
    \text{s.t.}\qquad\quad h \in \cH_0(S),\\
     h' \in \cH_0(S),\\
     h(z') \neq h'(z').
\end{align*}
\label{lemma: reliability radius objective}
\end{lemma}
\begin{proof}
    Let $r$ be the largest reliability radius of a test point $z$ that is if we perturb $z$ by a radius at most $r$ then the perturbed point is still in the agreement region. Consider for any perturbation $z'$ that there exists $h,h'$ that $h'(z')$ has a different label from $h'(z)$. If $\cH$ satisfies Assumption \ref{assumption: flexible class}, $h'(z) \neq h'(z')$ implies that the line between $z,z'$ must pass through the disagreement region. Therefore, $r \leq \lVert z - z' \rVert.$ On the other hand, let $r_0^2$ be the solution of the optimization given above. This implies that for any point $z'$ that $\lVert z - z' \rVert < r_0$, for any $h,h' \in \cH_0(S)$, we must have that $h(z) = h(z')$ that is $z' \in \operatorname{Agree}(\cH_0(S))$. Therefore, we have $ r_0 \leq r$ and we can conclude that $r = r_0$.
\end{proof}

\textbf{Remark. } This could be \edit{implemented} in practice. For example, in the case of linear separators, the problem takes the form of solving a \edit{constrained optimization} for each test point \edit{where many numerical solvers are readily available}. We observe that our proof of Theorem \ref{thm: SR-l1l2l3} above suggests an alternate margin-based approach which might be even more practical to implement, while retaining high probability reliability guarantees under the distributional assumption. If $\hat{h}$ denotes an ERM classifier on sample $S$, one could check the membership $z \in \hat{A}_\eta: = \{x: \lVert x\rVert_2 < \alpha \sqrt{d} - \eta\} \cap \{ x: |\langle w_{\hat{h}},x \rangle| \ \geq C_1\alpha\varepsilon\sqrt{d}+ \eta\}$ for any $\eta\ge 0$ by just computing the norm and margin w.r.t.\ $\hat{h}$. A simple halving search (e.g.\ starting with $\eta=\frac{1}{2}$) could be used to estimate the reliability radius.

Further, we can relax this constrained objective into a regularized objective that can be solved using empirical risk minimization. In the following Lemma, we show that this provides a lower bound on the reliability radius. 
 \begin{lemma}
     (Relaxation of the optimization objective for reliability radius) 
      Let $\cH$ be a hypothesis class and $\cD$ be a distribution over $\R^d$. If $\cH$ satisfies Assumption \ref{assumption: flexible class} then for a set of samples $S\sim \cD^m$, let $h_1,h_2,z^*$ be the optimal solution of the objective
    \begin{align*}
    h_1,h_2,z^* = \argmin_{h,h',z'} \lVert z-&z'\rVert^2 + \lambda(\hat{R}(h, S \cup \{(z',0)\}) + \hat{R}(h', S \cup \{(z',1)\}))
    \end{align*}
when $\hat{R}(h,A)$ is an empirical risk of $h$ on the sample $A$ then $\lVert z-z^*\rVert^2 \leq \edit{r^2}$ when $r$ the reliability radius of $z$.
\label{lemma: relax reliability radius}
 \end{lemma}
 \begin{proof}
Let $h,h',z'$ be the optimal solution of the optimization objective in Lemma \ref{lemma: reliability radius objective} so that the reliability radius $r$ is given by $\lVert z - z' \rVert$. Without loss of generality, let $h(z') = 0$ and $h'(z') = 1$. By definition, we have $\hat{R}(h, S \cup \{(z',0)\}) = 0$ and $\hat{R}(h', S \cup \{(z',1)\})) = 0$. Let $h_1,h_2, z^*$ be an optimal solution of the objective in Lemma \ref{lemma: relax reliability radius} then we have 
\begin{align*}
&\lVert z-z^*\rVert^2 + \lambda(\hat{R}(h_1, S \cup \{(z^*,0)\}) + \hat{R}(h_2, S \cup \{(z^*,1)\}))\\
&\leq 
\lVert z-z'\rVert^2 + \lambda(\hat{R}(h, S \cup \{(z',0)\}) + \hat{R}(h', S \cup \{(z',1)\}))\\
&= \lVert z-z'\rVert^2
\end{align*}
Since the empirical risk is non-negative, we can conclude that $\lVert z-z^*\rVert^2 \leq \lVert z-z'\rVert^2 = r^2$.
 \end{proof}

Similar observations also apply to the computation of the safely-reliable region. It may be useful to compute the safely-reliable region when we want to estimate the robust reliability performance of an algorithm on test data. In particular, this may be helpful in determining how often and where our learner gives a bad reliability radius, which can inform its safe deployment in practice.

}

\section{Bounds on the $\ptoq$ disagreement coefficient}\label{appendix: bounds on the p to q disagreement coefficient}
We will now consider some commonly studied concept spaces, and bound the $\ptoq$ disagreement coefficient for broad classes of distribution shifts.

\paragraph{Linear separators and nearly log-concave or $s$-concave distributions.}

We give a bound on $\Theta_{\cP \to \cQ}$ for $\cP$ and $\cQ$ isotropic nearly log-concave  distributions \cite{applegate1991sampling} over $\R^d$, a broad class that includes isotropic log-concave distributions.

\begin{definition}[$\beta$-log-concavity]
    A density function $f : \R^d \rightarrow \R_{\ge0}$ is $\beta$-log-concave if for any $\lambda\in[0,1]$ and any $x_1,x_2\in\R^d$, we have $f(\lambda x_1 +(1-\lambda)x_2) \ge e^{-\beta}f(x_1)^\lambda f(x_2)^{1-\lambda}$.
\end{definition}

\noindent For example, 0-log-concave densities are also log-concave. Also, since the condition in the definition holds for $\lambda=0$, we have $\beta\ge0$. Note that a smaller value of $\beta$ makes the distribution closer to logconcave. 
We have the following bound on $\Theta_{\cP \to \cQ}$ for distribution shift involving nearly log-concave densities. At a high level, we bound the angle between the normal vectors of the linear separators with small disagreement with $h^*$ under $\cP$, and bound the probability mass of their disagreement region under $\cQ$, by refining and generalizing the arguments from \cite{balcan2013active}. In particular, we quantify how much near-logconcavity is sufficient for the angle bound to hold, which further implies that any point in the disagreement region is either far away from the mean or close to the margin.

\begin{theorem}\label{thm:pq-beta-concave}
    Let the concept space $\cH$ be the class of linear separators in $\R^d$. Let $\cP$ be isotropic $\beta_1$-log-concave and $\cQ$ be isotropic $\beta_2$-log-concave, over $\R^d$. Then for  $0\le \beta_1,\beta_2 \le \frac{1}{56\lceil\log_2(d+1)\rceil}$, we have $\Theta_{\cP \to \cQ}(\varepsilon)= O(d^{1/2+ \frac{\beta_2}{2\ln 2}} \log(d/\varepsilon)).$
\end{theorem}

\begin{proof}

    Our proof builds on and generalizes the arguments used in the proof of Theorem 14 in \cite{balcan2013active}.
    Let $h\in B_\cP(h^*,r)$, i.e. $d(h,h^*)\le r$. We can apply the whitening transform from Theorem 16 of \cite{balcan2013active} provided $(1/20+c_1)\sqrt{1/C_1-c_1^2}\le 1/9$, where $C_1=e^{\beta_1\lceil\log_2(d+1)\rceil}$ and $c_1=e(C_1-1)\sqrt{2C_1}$. It may be verified that this condition holds for $0\le \beta_1\le \frac{1}{56\lceil\log_2(d+1)\rceil}$.  Now, by Theorem 11 of \cite{balcan2013active} we can bound the angle between their normal vectors as $\theta(w_h,w_{h^*}) \le cr$ where $c$ is an absolute constant. Now if $x\in\cX$ has a large margin $|w_{h^*} \cdot x| \ge cr\alpha $ and small norm $||x||\le \alpha$, for some $\alpha>0$, we have
    $$|w_h \cdot x - w_{h^*}\cdot x| \le ||w_h-w_{h^*}||\cdot||x|| < cr \alpha.$$
Now the large margin condition $|w_{h^*} \cdot x| \ge cr\alpha $ implies $\langle w_h,x\rangle \langle w_{h^*},x\rangle > 0,$ or $h(x)=h^*(x)$. Since $h\in B_\cP(h^*,r)$ was arbitrary, we have $x\notin\operatorname{DIS}(B_\cP(h^*,r)))$. Therefore, the set $\{x\mid ||x||> \alpha\}\cup \{x\mid |w_{h^*} \cdot x| \le cr\alpha\}$ contains the disagreement region $\operatorname{DIS}(B_\cP(h^*,r)))$. 
    
    By Theorem 11 of \cite{balcan2013active}, since $\cQ$ is an isotropic $\beta_2$-log-concave distribution, we have $\Pr_Q[||x|| > R\sqrt{Cd}] < Ce^{-R+1}$, for $C = e^{\beta_2\lceil \log_2(d+1)\rceil}$. Thus setting $\alpha=\sqrt{Cd}\log\frac{\sqrt{C}}{r}$ gives    $\Pr_\cQ[||x|| > \alpha] < e\sqrt{C}r$. Also, by Theorem 11 of \cite{balcan2013active}, for sufficiently small non-negative $\beta_2\le \frac{1}{56\lceil\log_2(d+1)\rceil}$, we have $\Pr_Q [|w_{h^*} \cdot x| \le cr\alpha]\le c'r\sqrt{Cd}\log\frac{\sqrt{C}}{r}$ for constant $c'$. The proof is concluded by a union bound and applying Definition \ref{def-pq}.
\end{proof}

 \noindent 

We further consider the case where the distributions belong to the broad class of isotropic $s$-concave distributions. In particular, unlike $\beta$-log-concave distributions, the distributions from this class can potentially be fat-tailed.

\begin{definition}[$s$-concavity]
    A density function $f : \R^d \rightarrow \R_{\ge0}$ is $s$-concave for $s\in(-\infty,1]\cup\{-\infty\}$ if for any $\lambda\in[0,1]$ and any $x_1,x_2\in\R^d$, we have $f(\lambda x_1 +(1-\lambda)x_2) \ge (\lambda f(x_1)^s+{(1-\lambda)}f(x_2)^s)^{1/s}$.
\end{definition}

\noindent Note that any $s$-concave function is  also $s'$-concave if $s >s'$. Moreover, concave functions are $1$-concave and log-concave functions are $s$-concave for any $s<0$. Using results from \cite{balcan2017sample}, we adapt the arguments in Theorem \ref{thm:pq-beta-concave} to show a bound on the disagreement coefficient when $\cP$ is isotropic $\beta$-log-concave and $\cQ$ is isotropic $s$-concave. 

\begin{theorem}\label{thm:pq-s-concave}
    Let the concept space $\cH$ be the class of linear separators in $\R^d$. Let $\cP$ be isotropic $\beta$-log-concave and $\cQ$ be isotropic $s$-concave, over $\R^d$. Then for \edit{$ 0 > s \ge -1/(2d +3)$} and sufficiently small non-negative $\beta \le \frac{1}{56\lceil\log_2(d+1)\rceil}$, we have $\Theta_{\cP \to \cQ}(\varepsilon)= O\left(\sqrt{d}\frac{2(1+ds)^2}{s+s^2(d+2)}(1-\varepsilon^{s/(1+ds)})\right).$
\end{theorem}

\begin{proof}

Similar to the proof of Theorem \ref{thm:pq-beta-concave}, we can apply the whitening transform from Theorem 16 of \cite{balcan2013active} provided $(1/20+c_1)\sqrt{1/C_1-c_1^2}\le 1/9$, where $C_1=e^{\beta\lceil\log_2(d+1)\rceil}$ and $c_1=e(C_1-1)\sqrt{2C_1}$. It may be verified that this condition holds for $0\le \beta\le \frac{1}{56\lceil\log_2(d+1)\rceil}$. We can also show that the set $\{x\mid ||x||> \alpha\}\cup \{x\mid |w_{h^*} \cdot x| \le cr\alpha\}$ contains the disagreement region $\operatorname{DIS}(B_\cP(h^*,r)))$.

By Theorem 11 of \cite{balcan2017sample}, we have $\Pr_\cQ [|w_{h^*} \cdot x| \le cr\alpha]\le \frac{2(1+ds)}{1+s(d+2)}\cdot cr\alpha$.  By Theorem 5 of \cite{balcan2017sample}, since $\cQ$ is an isotropic $s$-concave distribution, we have  $\Pr_Q[||x|| > t\sqrt{d}] < \left(1-\frac{c_1st}{1+ds}\right)^{(1+ds)/s}$, for any $t\ge 16$ and absolute constant $c_1$. This implies $\Pr_\cQ[||x|| > c_1\sqrt{d}\frac{1+ds}{s}(1-r^{s/(1+ds)})] < Cr$ for some constant $C$. Thus setting $\alpha=c_1\sqrt{d}\frac{1+ds}{s}(1-r^{s/(1+ds)})$ gives    $\Pr_\cQ[||x|| > \alpha] < Cr$. Also, for this $\alpha$, we have $\Pr_\cQ [|w_{h^*} \cdot x| \le cr\alpha]\le c\frac{2(1+ds)}{1+s(d+2)}\cdot c_1\sqrt{d}\frac{1+ds}{s}(1-r^{s/(1+ds)})r = c'\sqrt{d}\frac{2(1+ds)^2}{s+s^2(d+2)}(1-r^{s/(1+ds)})r$ for constant $c'$. The proof is concluded by a union bound and applying Definition \ref{def-pq}.
\end{proof}

\paragraph{Smooth classification boundaries.}
We also illustrate our notion for more general concept spaces beyond linear separators. Specifically, we consider classifiers with smooth boundaries (Definition \ref{def:smooth-classifiers}). \looseness-1
\noindent If we further assume that the probability density may be upper and lower bounded by an $\alpha$th order smooth function, we can bound the disagreement coefficient for shift from $\cP$ to $\cQ$. Interestingly, while \cite{wang2011smoothness} need the distribution to be sandwiched between smooth functions, our result only needs a lower bound on the  smoothness of $\cP$ and an upper bound on the smoothness of $\cQ$.  
\begin{theorem}
\label{thm:pq-smooth}
    Let the instance space be $\cX = [0,1]^{d+1}$. Let the hypothesis space be $\cH_\alpha^C$, with $d < \alpha<\infty$. If the marginal distributions $\cP_\cX,\cQ_\cX$ have densities $p(x)$ and $q(x)$ on $[0,1]^{d+1}$ such that there exists an $\alpha$th order smooth function $g(x)$ and $a_p, b_q\in\R_+$ such that $a_pg(x)\le p(x) $ and $ q(x) \le b_qg(x)$ for all $x \in[0,1]^{d+1}$, then $\Theta_{\cP \to \cQ}(\varepsilon)= O\left({b_q}{a_p^{-\frac{\alpha}{d+\alpha}}}\varepsilon^{-\frac{d}{d+\alpha}}\right).$
\end{theorem}

\begin{proof}

    We will extend the arguments from \cite{wang2011smoothness} to the distribution shift setting. Let $\x=(x_1,\dots,x_d)\in[0,1]^d$ and let $h\in B_P(h^*,r)$ where $h^*\in\cH_\alpha$ is the target concept. Denote $\Phi_h^p(\x)=\int_{f_{h^*}(\x)}^{f_h(\x)}p(\x,x_{d+1})dx_{d+1}$, and $\Phi_h^q(\x)=\int_{f_{h^*}(\x)}^{f_h(\x)}q(\x,x_{d+1})dx_{d+1}$. It is easy to verify by taking derivatives that $\tilde{\Phi}_h(\x)=\int_{f_{h^*}(\x)}^{f_h(\x)}g(\x,x_{d+1})dx_{d+1}$ is $\alpha$th order smooth. Since $h\in B_P(h^*,r)$,
    $$\int_{[0,1]^d}|\tilde{\Phi}_h(\x)|d\x\le \int_{[0,1]^d}\frac{1}{a_p}|{\Phi}_h^p(\x)|d\x\le \frac{r}{a_p}.$$
    By Lemma 11 of \cite{wang2011smoothness}, this implies $||\tilde{\Phi}_h||_{\infty}=O\left((\frac{r}{a_p})^{\frac{\alpha}{d+\alpha}}\right)$, and therefore $||{\Phi}_h^q||_{\infty}\le b_q||\tilde{\Phi}_h||_{\infty}=O\left({b_q}{a_p^{-\frac{\alpha}{d+\alpha}}}r^{\frac{\alpha}{d+\alpha}}\right)$. Since this holds for any $h\in B_\cP(h^*,r)$, we have
    $$\sup_{h\in B_P(h^*,r)}||{\Phi}_h^q||_{\infty}=O\left({b_q}{a_p^{-\frac{\alpha}{d+\alpha}}}r^{\frac{\alpha}{d+\alpha}}\right).$$
    By definition of region of disagreement, we have
    \begin{align*}
        \Pr_{x\sim \cQ_\cX}[x\in DIS(B_\cP(h^*,r))]&=\Pr_{x\sim \cQ_X}[x\in\cup_{h\in B_\cP(h^*,r)}\{x'\mid h(x')\ne h^*(x')\}]\\
        &\le 2\int_{[0,1]^d}\sup_{h\in B_\cP(h^*,r)}||{\Phi}_h^q||_{\infty}d\x\\
        &=O\left({b_q}{a_p^{-\frac{\alpha}{d+\alpha}}}r^{\frac{\alpha}{d+\alpha}}\right).
    \end{align*}
    The result follows from definition of $\Theta_{\cP \to \cQ}(\epsilon)$.
\end{proof}

\section{Simple examples for the $\ptoq$ disagreement coefficient}
\label{appendix: disagreement coeff}
\textbf{Example 1.} (Non-overlapping spheres the same center) Let $\cP$, $\cQ$ be uniform distribution over a sphere with the center at the origin with radius $1$ and $2$ respectively. Let $\cH$ be a class of linear separators that pass through the origin and $h^* \in \cH$. By symmetry, we have
\begin{align*}
    \Theta_{\cP \to \cQ}(\varepsilon) &= \sup_{r\geq \varepsilon}\frac{\Pr_
    \cQ(\operatorname{DIS}(B_\cP(h^*,r)))}{r}\\
    &= \sup_{r\geq \varepsilon}\frac{\Pr_\cP(\operatorname{DIS}(B_\cP(h^*,r)))}{r}\\
    &= \Theta_{\cP}(\varepsilon).
\end{align*} The disagreement coefficient from $\cP$ to $\cQ$ is the same as the disagreement coefficient on $\cP$.

\textbf{Example 2.} (Thresholds) Let $\cP$, $\cQ$ be uniform distribution over an interval $[-\frac{1}{2}, \frac{1}{2}]$ and $[-1, 1]$ respectively. Let $\cH$ be a class of a threshold function. Let $h^*$ have a thereshold at $0$. We have $\operatorname{DIS}(B_\cP(h^*,r)) = [-r, r]$ and 
\begin{align*}
    \Theta_{\cP \to \cQ}(\varepsilon) &= \sup_{r\geq \varepsilon}\frac{\Pr_\cQ([-r,r])}{r}\\
    &= \frac{\frac{2r}{2}}{r} = 1
\end{align*}
compared to the disagreement coefficient $\Theta_{\cP}(\varepsilon) = 2$.

\section{Additional proof details for distribution shift}\label{appendix: distribution shift}

\begin{proof} (of Theorem \ref{thm:pq-reliability}) 
    If $S\sim \cP^m$, with $m\ge\frac{c}{\epsilon^2}(d+\ln\frac{1}{\delta})$ for some sufficiently large constant $c$, we have by uniform convergence (\cite{anthony1999neural}, Theorem 4.10) that with probability at least $1 -\delta$, we have $d_\cP(h,h^*) \le d_S(h,h^*)+\epsilon$ for all $h \in\cH$. Here $d_\cD(h_1,h_2)=\Pr_{x\sim\cD_X}[h_1(x)\ne h_2(x)]$, and $d_S(h_1,h_2)=\frac{1}{|S|}\sum_{x\in S}\bbI[h_1(x)\ne h_2(x)]$. Therefore, $\operatorname{Agree}(B_\cP^\cH(h^*,\epsilon))\subseteq \operatorname{Agree}(\cH_0(S))\subseteq R^\cL(S)$ in this event. Denoting this event by `$E$' and its complement by `$\Bar{E}$', we have

    \begin{align*}
        \Pr_{x\sim \cQ, S\sim \cP^m}[x\in R^\cL(S)]&=\Pr_{x\sim \cQ}[x\in R^\cL(S)\mid E]\Pr[E]+\Pr_{x\sim \cQ}[x\in R^\cL(S)\mid \Bar{E}]\Pr[\Bar{E}]\\
        &\ge \Pr_{x\sim \cQ}[x\in R^\cL(S)\mid E]\cdot (1- \delta)\\
        &\ge \Pr_{x\sim \cQ}[x\in R^\cL(S)\mid E]-\delta\\
        &\ge \Pr_{x\sim \cQ}[x\in \operatorname{Agree}(B_\cP^\cH(h^*,\epsilon))]-\delta.
    \end{align*}

Noting $\Pr_{x\sim \cQ}[x\in \operatorname{Agree}(B_\cP^\cH(h^*,\epsilon))]=1-\Pr_{x\sim \cQ}[x\in \operatorname{DIS}(B_\cP^\cH(h^*,\epsilon))]$ and using Definition \ref{def-pq} completes the proof.
\end{proof}

\section{Safely-reliable correctness under distribution shift \red{or Reliable robustness transfer}}\label{sec:robustness-transfer}
There is a growing practical \cite{shafahiadversarially,salman2020adversarially} as well as recent theoretical interest \cite{deng2023hardness} in the setting of `robustness transfer', where one simultaneously expects adversarial test-time attacks as well as distribution shift. We will study the reliability aspect for this more challenging setting. We note that the definition of a robustly-reliable learner does not depend on the data distribution (see Definition \ref{def:robustly-reliable-metric}) as the guarantee is pointwise. Our optimality result in Section \ref{sec: rr-w-ball} applies even when a test point is drawn from a different distribution $\cQ$. In this case, the safely-reliable region instead would have a different probability mass. 

\begin{definition}[$\ptoq$ safely-reliable correctness]\label{def:pqrc}
The $\ptoq$ safely-reliable correctness of $\cL$ (at sample rate $m$, for distribution shift from $\cP$ to $\cQ$, w.r.t.\ robust loss $\ell$) is defined as the probability mass of its safely-reliable region under $Q$, on a sample $S\sim\cP^m$, i.e. $\text{PQR}_\ell^\cL(S,\eta_1,\eta_2):=\Pr_{x\sim Q, S\sim P^m}[x\in \text{SR}^{\cL}_\ell(S, \eta_1,\eta_2)]$.
\end{definition}

\noindent We will now combine our results on test-time attacks and distribution shift to give a general bound on the $\ptoq$ safely-reliable correctness for the  different robust losses (Definition \ref{def:losses}).

\begin{theorem}\label{thm:pq-robust-reliability}
Let $\cQ$ be a realizable distribution shift of $\cP$ with respect to $\cH$, and $h^*\in\cH$ be the target concept.
There exist learners for robust losses $\ellca,\elltl,\ellst, \ellia$ with $\ptoq$ safely-reliable correctness given by 
\begin{enumerate}[label=(\alph*),leftmargin=*,topsep=0pt,partopsep=1ex,parsep=1ex]\itemsep=-4pt
    \item $\text{PQR}_{\text{CA}}^\cL(S,\eta_1,\eta_2)=\Pr_{x\sim Q}[\B_\cM(x,\eta_1)\cap \{z\mid h^*(z) = h^*(x) \} \subseteq \operatorname{Agree}(\cH_0(S))],$
    \item $\text{PQR}_{\text{TL}}^\cL(S,\eta_1,\eta_2)=\Pr_{x\sim Q}[\B_\cM(x,\eta_1) \subseteq \operatorname{Agree}(\cH_0(S))],$
    \item $\text{PQR}_{\text{ST}}^\cL(S,\eta_1,\eta_2)=\Pr_{x\sim Q}[\B_\cM(x,\eta_1) \subseteq\{z\mid\B^o(z,\eta_2)\subseteq\operatorname{Agree}(\cH_0(S)) \wedge h(x)=h(z), \;\forall\;x\in \B^o(z,\eta_2), h\in \cH_0(S)\}],$
    \item \edit{ {$\text{PQR}_{\text{IA}}^\cL(S,\eta_1,\eta_2)
= \Pr_{x\sim Q}\!\left[
    \bigl(h^*(x)=0 \land \B(x,\eta_1)\subseteq A_{\text{IA}}\bigr)
    \lor
    \bigl(h^*(x)=1 \land x\in A_{\text{IA}}\bigr)
\right]$.}}
\end{enumerate}

\end{theorem}

\begin{proof}
    The proof follows by applying Theorems \ref{thm:rr-eta-l1} and \ref{thm:mball-l3-ub}, and using Definitions \ref{def:srr} and \ref{def:pqrc}.
\end{proof}

\noindent We consider an example when the training distribution $\cP$ is isotropic log-concave and the test distribution $\cQ_\mu$  is  log-concave with its mean shifted by $\mu$ but the covariance matrix is still an identity matrix (see Figure \ref{fig: distribution shift}, right).

\begin{theorem} Let $\cP,\cQ$ be isotropic log-concave over $\R^d$.  Let $\cQ_\mu$ be a distribution after shifting the mean of $\cQ$ by $\mu \in \R^d$. Let $\cH = \{h: x \to \operatorname{sign}(\langle w_h, x\rangle ) \mid w_h \in \R^d, \lVert w_h \rVert_2 = 1\}$ be the class of linear separators. Let $\B(\cdot, \eta)$ be a $L_2$ ball perturbation with radius $\eta$. For $S \sim \cP^m$, for  $m = \cO(\frac{1}{\varepsilon^2}(\operatorname{VCdim}(\cH) + \ln\frac{1}{\delta}))$,  for an optimal robustly-reliable learner $\cL$, we have
\begin{enumerate}[label=(\alph*),leftmargin=*,topsep=0pt,partopsep=1ex,parsep=1ex]\itemsep=-4pt
    \item $ \Pr_{\cQ_\mu}(\srtl) \geq 1 -  2(\eta_1 + \lVert \mu \rVert_2) - \Tilde{\cO}(\sqrt{d}\varepsilon)$ with probability   $1-\delta$,
    \item $\srca = \srtl$,
    \item $\Pr_{\cQ_\mu}(\srst) \geq  1 - 2(\eta_1 + \eta_2 + \lVert \mu \rVert_2) - \Tilde{\cO}(\sqrt{d}\varepsilon)$  with probability    $1-\delta$,
    \item \blue{$\Pr_{\cQ_\mu}(\sria) \geq  1 - (\eta_1 + \eta_2 + 2\lVert \mu \rVert_2) - \Tilde{\cO}(\sqrt{d}\varepsilon)$  with probability    $1-\delta$.}
\end{enumerate}

The $\Tilde{\cO}$-notation suppresses dependence on logarithmic factors and distribution-specific constants.
\label{thm: cross-proeduct linear separators}
\end{theorem}
\begin{proof}
From triangle inequality, we know that $\B(x - \mu, r + \lVert \mu \rVert_2) \supseteq \B(x,r)$. We can simply extend the proofs from Theorem \ref{thm: SR-l1l2l3}. Recall that 
$\srca = \srtl = \{ x \mid \B(x,\eta_1) \subseteq \operatorname{Agree}(\cH_0(S)) \} \supseteq \{ x \mid \B(x - \mu,\eta_1 + \lVert \mu \rVert_2) \subseteq \operatorname{Agree}(\cH_0(S)) \}$ and $\srst = \{x \in \cX \mid  \B_\cM(x,\eta_1 + \eta_2) \subseteq\operatorname{Agree}(\cH_0(S))\} \supseteq \{x \in \cX \mid  \B_\cM(x - \mu,\eta_1 + \eta_2 + \lVert \mu \rVert_2) \subseteq\operatorname{Agree}(\cH_0(S))\}.$   When $x$ is drawn from a distribution $\cQ_\mu$, we know that $x - \mu$ follows a distribution $\cQ$ which is isotropic log-concave.  We can apply Lemma \ref{lemma: ball in agreement} to bound the probability mass of the safely-reliable region under $\cQ_\mu$. \edit{With probability at least $1-\delta$ over $S\sim \cP^m$, uniform convergence and isotropic log-concavity of $P$ give the angular control on $\cH_0(S)$ used in Lemma \ref{lemma: ball in agreement}. We then apply the norm-tail and anti-concentration bounds under $\cQ$ to bound the probability mass of the resulting region.} \blue{Similarly, we can do the same for $\sria$.}
\end{proof}

\noindent Similar bounds on reliability under robustness transfer may be given for linear separators under more general source or target distributions, including isotropic $\beta$-log-concave or $s$-concave distributions, as well as for concept classes with smooth classification boundaries, by applying Theorem \ref{thm:pq-robust-reliability} to the examples from previous sections.

\cready{
\section{Agnostic setting}\label{app:agnostic}

\begin{proof}[Proof of Theorem \ref{thm:rr-eta-l1-agnostic}] The robustly-reliable learner $\cL$ is given as follows. 
 Set $h^{\cL}_S=\argmin_{h\in\cH}\operatorname{err}_S(h)$ i.e. an ERM over $S$, and $r^{\cL}_S(z)=\infty$ if $z\in\operatorname{Agree}(\cH_\nu(S))$, else $r^{\cL}_S(z)=-1$. To study the robustly-reliable region, we assume there is some concept $h^*\in\cH$ which satisfies $\err_S(h^*)\le \nu$.
 By definition of ERM, $\operatorname{err}_S(h^{\cL}_S)\le \operatorname{err}_S(h^*)=\nu$, or $h^{\cL}_S\in\cH_\nu(S)$.  We first show that $\cL$ is robustly-reliable. For $z\in\cX$, if $r^{\cL}_S(z)=\eta> 0$, then $z\in \operatorname{Agree}(\cH_\nu(S))$. 
 We have $h^*(z)=h^{\cL}_S(z)$ since the classifiers $h^*,h^{\cL}_S\in \cH_\nu(S)$ and $z$ lies in the agreement region of classifiers in $\cH_\nu(S)$ in this case. Thus, we have $\ellca^{h^*}(h^{\cL}_S,x,z)=0$ for any $x$ such that $z\in\B_\cM^o(x,\eta)$. In the $\eta=0$ case, $h^*(z)=h^{\cL}_S(z)$ by definiton and the same argument applies. Therefore, $\rrca(S,\nu,\eta)\supseteq \operatorname{Agree}(\cH_\nu(S))$ for all $\eta\ge 0$ follows from the setting $r^{\cL}_S(z)=\infty$ if $z\in\operatorname{Agree}(\cH_\nu(S))$.

Conversely, let $z\in \operatorname{DIS}(\cH_\nu(S))$. There exist $h_1,h_2\in \cH_\nu(S)$ such that $h_1(z)\ne h_2(z)$. By definition, robustly-reliable learning with $\eta=0$ is not possible for $z$. If possible,  let there be a robustly-reliable learner $\cL$ such that $z\in \rrca(S,\nu,\eta)$ for some $\eta>0$. By definition of the robust-reliability region, we must have $r^{\cL}_S(z)> 0$. By definition of a  ball, we have $z\in\B_\cM^o(z,\eta)$ for any $\eta>0$, and therefore $\ellca^{h^*}(h^{\cL}_S,z,z)=0$ for every $h^*\in\cH$ such that $\err_S(h^*)\le\nu$. But then we must have $h^{\cL}_S(z)=h^*(z)$ by definition of $\ellca$. But we can set $h^*=h_1$ or $h^*=h_2$ since both are in $\cH_\nu(S)$. But $h_1(z)\ne h_2(z)$, and therefore $h^{\cL}_S(z)\ne h^*(z)$ for one of the above choices for $h^*$, contradicting that $\cL$ is robustly-reliable.  
\end{proof}
}

\end{document}